\documentclass[conference]{IEEEtran}
\IEEEoverridecommandlockouts

\usepackage[most]{tcolorbox}
\usepackage{amsmath}
\usepackage{amssymb}
\usepackage{algorithm}
\usepackage{algorithmic}
\usepackage{graphicx}
\usepackage{bibarts}
\usepackage{cite}
\usepackage{amsthm}
\usepackage{subfigure}
\usepackage{diagbox}
\usepackage{caption}
\usepackage{color}
\usepackage{bbm}
\usepackage{dsfont}
\usepackage[font=small, labelfont=bf, labelsep=period]{caption}

\title{\huge Optimizing Model Splitting and Device Task Assignment for Deceptive Signal Assisted Private Multi-hop Split Learning  \vspace{-0.3cm}}

\author{
    \IEEEauthorblockN{
        Dongyu Wei, Xiaoren Xu, Yuchen Liu, \IEEEmembership{Member, IEEE},\\
        H. Vincent Poor, \IEEEmembership{Life Fellow, IEEE}, and Mingzhe Chen, \IEEEmembership{Senior Member, IEEE}
        \thanks{Dongyu Wei, Xiaoren Xu, and Mingzhe Chen are with the Department of Electrical and Computer Engineering, University of Miami, Coral Gables, FL, 33146, USA (e-mail: {dongyu.wei, xiaoren.xu, mingzhe.chen}@miami.edu).}%
        \thanks{Yuchen Liu is with the Department of Computer Science, NC State University, Raleigh, NC, 27695, USA (e-mail: yuchen.liu@ncsu.edu).}%
        \thanks{H. Vincent Poor is with the Department of Electrical and Computer
        Engineering, Princeton University, Princeton, NJ, 08544, USA (e-mail:
        poor@princeton.edu).}
        \thanks{Mingzhe Chen is also with the Frost Institute for Data Science and Computing, University of Miami, Coral Gables, FL 33146 USA.}
        \vspace{-0.7cm}
    }
}

\begin{document}
\theoremstyle{definition}
\newtheorem{theorem}{Theorem}
\newtheorem{proposition}{Proposition}
\newtheorem{lemma}{Lemma}
\newtheorem{corollary}{Corollary}

\maketitle
\vspace{-1em}

\begingroup
\renewcommand\thefootnote{}\footnote{
This paper was presented in part at the 2025 Annual Conference on Information Sciences and Systems (CISS) as the paper \cite{10944721}.
}
\addtocounter{footnote}{-1}
\endgroup
\begin{abstract}
In this paper, deceptive signal-assisted private split learning is investigated. In our model, several edge devices jointly perform collaborative training, and some eavesdroppers aim to collect the model and data information from devices. To prevent the eavesdroppers from collecting model and data information, a subset of devices can transmit deceptive signals. Therefore, it is necessary to determine the subset of devices used for deceptive signal transmission, the subset of model training devices, and the models assigned to each model training device. This problem is formulated as an optimization problem whose goal is to minimize the information leaked to eavesdroppers while meeting the model training energy consumption and delay constraints.  
To solve this problem, we propose a soft actor-critic deep reinforcement learning framework with intrinsic curiosity module and cross-attention (ICM-CA) that enables a centralized agent to determine the model training devices, the deceptive signal transmission devices, the transmit power, and sub-models assigned to each model training device without knowing the position and monitoring probability of eavesdroppers. The proposed method uses an ICM module to encourage the server to explore novel actions and states and a CA module to determine the importance of each historical state-action pair thus improving training efficiency.
Simulation results demonstrate that the proposed method improves the convergence rate by up to $3 \times$ and reduces the information leaked to eavesdroppers by up to $13 \%$ compared to the traditional SAC algorithm.
\end{abstract}


\section{Introduction}\label{Introduction}
Federated learning (FL) has emerged as a prominent approach for privacy-enhancing and distributed model training \cite{chen2020joint, 10960614}. However, FL requires each edge device used in distributed training to transmit its entire locally trained model (e.g., gradient or model parameters) to a central server for model aggregation, which may not be practical due to limited computational resources, energy and communication constraints\cite{chen2021distributed}. To address this limitation, split learning (SL), which partitions a global model into multiple sub-models and assigns them to edge devices for collaborative training, has been proposed. 
However, deploying SL in wireless networks introduces several critical challenges. First, determining the optimal partitioning of the global model into sub-models for each edge device is non-trivial, as it must account for the heterogeneous computational capacities, energy consumption constraints, and communication capabilities of devices. Moreover, although model partitions across different devices in SL can reduce the likelihood of potential eavesdroppers obtaining the complete models, the frequent device-to-device model information exchange introduces new security and privacy risks, as eavesdroppers can exploit these transmissions to infer model information \cite{10320405, 10700751, 10516589}. 

A number of existing works \cite{xu2023accelerating,liu2022wireless,lin2024efficient,wu2023split,yang2023over,kim2024splitmac} have studied the optimization of model partitioning, communication efficiency, and resource management for SL in wireless networks. In particular, the authors in \cite{xu2023accelerating} investigated the combination of SL and FL to optimize model partitioning and bandwidth allocation, thus reducing training delay in wireless networks. Similarly, the work in \cite{liu2022wireless} proposed a hybrid approach using which model updates are selectively transmitted based on channel conditions and computational constraints, thus balancing training efficiency and communication cost. The work in \cite{lin2024efficient} introduced a parallel SL framework that enhances training efficiency by leveraging parallel client updates while optimizing communication overhead through dynamic device selection and bandwidth allocation. In \cite{wu2023split}, the authors proposed a resource-aware SL framework that clusters devices based on their computing capabilities and optimizes the cut layer selection and wireless resource allocation to minimize overall training time. The study in \cite{yang2023over} explored the integration of SL with over-the-air computation (OAC) in multiple-input multiple-output (MIMO) networks to enhance communication efficiency and reduce model transmission delay, thereby improving overall training speed. The work in \cite{kim2024splitmac} focused on multi-device SL over multiple access channels, and introduced a communication scheme that allows simultaneous transmission of intermediate features to improve communication efficiency and bandwidth utilization. 
However, most of these works \cite{xu2023accelerating, lin2024efficient, wu2023split,yang2023over} assumed a fixed model partitioning scheme throughout the training process, without considering dynamic model splitting based on device computational capabilities and network conditions. Meanwhile, these works \cite{xu2023accelerating, lin2024efficient, wu2023split, liu2022wireless, yang2023over, kim2024splitmac} primarily considered a single-device-to-server transmission model, and hence, the designed schemes cannot be applied for multi-device SL frameworks. Finally, none of the existing works \cite{xu2023accelerating, liu2022wireless, lin2024efficient, wu2023split, yang2023over, kim2024splitmac} considered the risk of SL model eavesdropping during training, which poses a significant security threat as intermediate model features can be intercepted and exploited by adversaries.

Several existing works~\cite{al2022toward,lu2019blockchain,zhao2020privacy,lu2024split,weng2019deepchain,fang2019entrapment} have studied the design of privacy-preserving or secure distributed learning algorithms. In particular, the authors in \cite{al2022toward} used a deep reinforcement learning (RL) algorithm to dynamically identify and exclude unreliable devices in FL, so as to improve FL training robustness under adversarial threats. The work in\cite{lu2019blockchain} developed a blockchain-assisted FL framework to protect industrial IoT data by designing a secure and verifiable model aggregation scheme. The work in \cite{zhao2020privacy} introduced a privacy-preserving FL scheme that uses blockchain and homomorphic encryption to prevent information leakage during model transmission. In~\cite{lu2024split}, a split aggregation protocol was proposed to enhance Byzantine resistance in the communication process of FL, while maintaining low communication overhead and ensuring the transmitted model segments remain unchanged by attackers. The work in~\cite{weng2019deepchain} used blockchain to design a reward mechanism that can support privacy-preserving FL training. The work in~\cite{fang2019entrapment} used a deception-based defense mechanism that creates entrapment zones to confuse eavesdroppers and reduce the likelihood of successful data interception over wireless channels.
However, most of these works either focused on protecting the model aggregation process~\cite{lu2019blockchain,lu2024split}, or on preventing untrusted devices from uploading local updates~\cite{al2022toward,weng2019deepchain}. These defense methods are passive methods that exclude untrusted devices or their updates only when the parameter server detects suspicious or malicious information. The work in~\cite{fang2019entrapment} designed a proactive physical-layer deception strategy to confuse wireless eavesdroppers. However, it focused on a scenario where the entire model is transmitted by one model training device, and did not consider how to split the machine learning model and assign them to multiple edge devices according to data and model leakage risks, which can significantly reduce overall amount of information leaked to eavesdroppers since different ML model layers have different data leakage risks~\cite{burda2019large}. 

The main contribution of this work is a novel secure and efficient multi-hop split learning (MHSL) framework, which enables several devices to collaboratively train a large sized ML model with minimum model leakage risks, as well as limited energy and latency. The key contributions include the following:
\begin{enumerate}
    \item We consider a multi-hop split learning framework where a subset of devices is selected to jointly perform collaborative training, and another subset of devices is selected to transmit deceptive signals to prevent eavesdroppers from intercepting model training information. Due to limited energy and wireless spectrum, the number of selected model training and deceptive signal transmission devices is limited. Hence, it is necessary to jointly optimize ML model splitting and assignment, device task assignment (i.e., model training or deceptive signal transmission), and transmit power allocation. This problem is formulated as an optimization problem whose goal is to minimize expected amount of information leaked to eavesdroppers while meeting latency and energy requirements. 
    
    \item To address this problem, we propose a soft actor-critic (SAC) deep reinforcement learning framework with intrinsic curiosity module and cross-attention (ICM-CA). The designed method enables a central controller to select training and deceptive devices, split and assign the ML model, and determine transmit powers under energy and latency constraints without knowing the position and monitoring probability of eavesdroppers. Compared to standard RL methods \cite{10571594,10778255}, the proposed method leverages: 1) an intrinsic curiosity module to encourage the server to explore novel actions and states, thus discovering better policies that can further reduce the amount of information leaked to eavesdroppers, and 2) a cross-attention module that determines the importance of each historical state-action pair such that the designed RL method can focus more on the important state-action pairs, thus improving learning efficiency. 
    
    \item We analyze how the transmit powers of model training and deceptive signal transmission devices affect the amount of information leaked to eavesdroppers. We derive the optimal transmit powers that minimize the expected amount of information leaked to eavesdroppers under time and energy constraints. Analytical results show that both reducing the model transmit power and increasing the deceptive signal power will reduce the amount of information leaked to eavesdroppers.
\end{enumerate}
Simulation results show that the proposed ICM-CA approach improves the convergence rate by $3 \times$ and reduces the amount of information leaked to eavesdroppers by up to $13\%$ compared to the SAC algorithm.


\section{System Model and Problem Formulation}\label{Proposed_Clustered_FL_System}

\begin{figure}[tp]
    \centering
\includegraphics[width=.5\textwidth]{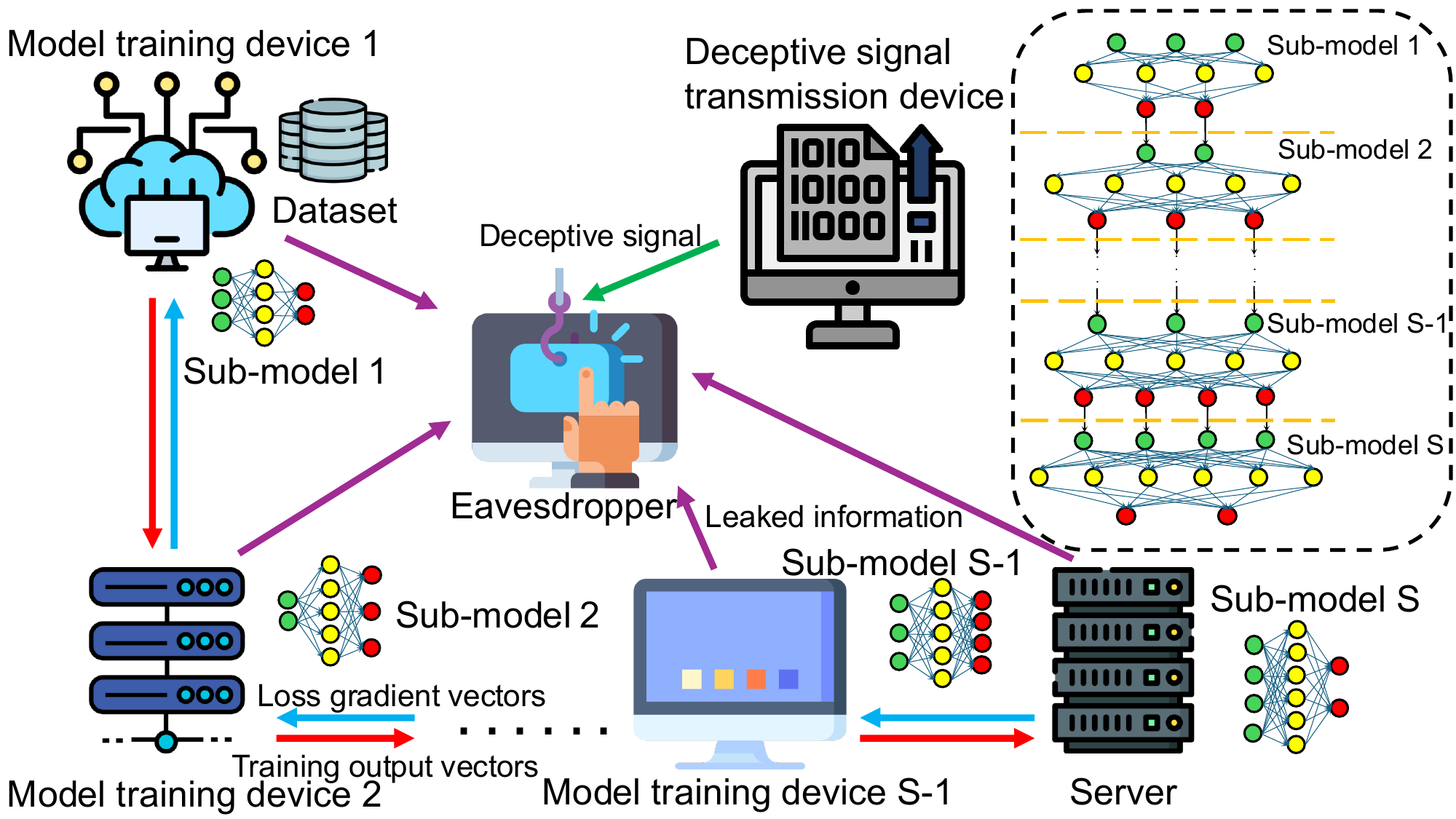}
\vspace{-0.2cm}
    \caption{The architecture of the considered MHSL model.}
    \label{fig:intro}
      \vspace{-0.7cm}
\end{figure}

Consider a distributed learning system where a set $\mathcal{U}$ of $U$ devices and a server cooperatively perform an MHSL algorithm, under the monitoring of a set $\mathcal{E}$ of $E$ eavesdroppers attempting to eavesdrop the ML and data information transmitted by the devices, as shown in Fig. \ref{fig:intro}. In our model, a subset $\mathcal{S}$ of $S-1$ devices will be selected to participate in the training of a global model parameterized by $\boldsymbol{\Theta}_{\textrm{G}}$. Specifically, the model $\boldsymbol{\Theta}_{\textrm{G}}$ will be divided into $S$ sub-models. These sub-models will be assigned to the selected $S-1$ devices and the server for model training. 
To reduce the data or model parameter leakage caused by eavesdroppers, the server will also select a subset of devices to send deceptive signals to the eavesdroppers such that the eavesdroppers will consider the deceptive signals as the useful information from devices, thus ignoring data or model information from the selected devices \cite{fang2019entrapment}. Our goal is to minimize the transmission delay and energy cost of model training while mitigating the data leakage caused by eavesdroppers.
Next, we introduce the considered MHSL algorithm and the corresponding training process. Then, we show the proposed attack model. Finally, we formulate our optimization problem.

\subsection{General Procedure of the MHSL Algorithm}
The general training process of the considered MHSL is summarized as follows:

    \textbf{1) Device Selection and Model Partition:} 
    The global model $\boldsymbol{\Theta}_{\textrm{G}}$ is partitioned into $S$ sequential sub-models $\boldsymbol{\Theta}_{\textrm{G}}=\left[\boldsymbol{\theta}_1, \boldsymbol{\theta}_2, \ldots, \boldsymbol{\theta}_S\right]$, where $\boldsymbol{\theta}_1$ is the starting segment of $\boldsymbol{\Theta}_{\textrm{G}}$, followed by $\boldsymbol{\theta}_2, \ldots, \boldsymbol{\theta}_S$. 
    The server will first determine the value of $S$. Then, the server will select $S-1$ devices and assigns the sub-models $\boldsymbol{\theta}_1, \boldsymbol{\theta}_2, \ldots, \boldsymbol{\theta}_{S-1}$ to the devices.

    \textbf{2) Loss Calculation:} 
    We assume that the device with model $\boldsymbol{\theta}_k$ is $s_k$. The device $s_1$ with model $\boldsymbol{\theta}_1$ needs to use its data to calculate its output. Other devices such as $s_k$ with model $\boldsymbol{\theta}_k$ will use the output of $s_{k-1}$ to calculate its output and send the output to $s_{k+1}$. Finally, the server with model $\boldsymbol{\theta}_S$ receives the output from $s_{S-1}$, and calculates the output and loss value of the entire model $\boldsymbol{\Theta}_{\textrm{G}}$.

    \textbf{3) Backward Propagation:} 
    The server will calculate the gradient vector of the model using the loss value and transmit it back to the selected devices. The transmission order is from the server to device $s_{S-1}$, and then from device $s_{S-1}$ to $s_1$.

    \textbf{4) Deceptive Signal Transmission:} 
    For each transmission of a gradient vector or an output of the sub-model, a subset of devices in $\mathcal{U}$ excluding the transmitter and receiver are selected to transmit deceptive signals.

    \textbf{5)}
    Repeat steps 2-4 until $\boldsymbol{\Theta}_{\textrm{G}}$ converges.

\subsection{Proposed MHSL Algorithm}
Let $\mathcal{X}_{s_1}$ be a set of $X_{s_1}$ data samples of device $s_1$. Next, we introduce the training process of the MHSL algorithm per iteration mathematically. 

    \textbf{1)Forward Propagation:} 
    Let $f_k(\cdot;\boldsymbol{\theta}_k)$ be the function implemented by user $s_k$. The output $\boldsymbol{z}_k$ of device $s_k$ at each iteration is
    \begin{equation}\label{eq:zk_all}
        \boldsymbol{z}_k = f_k\left(\boldsymbol{z}_{k-1}; \boldsymbol{\theta}_k\right), k = 1,\ldots,S,
    \end{equation}
    where $\boldsymbol{z}_0 = \mathcal{X}_{s_1}$, $\mathcal{X}_{s_1}$ is a mini-batch of $\mathcal{X}_{s_1}$, and $\boldsymbol{z}_{k-1}$ is the output of device $s_{k-1}$.
    From (\ref{eq:zk_all}), we see that device $s_1$ needs to use its data to calculate the output, while any other device $s_k$ uses the output from $s_{k-1}$ to calculate its output. 

    \textbf{2) Back Propagation and Model Update:} 
    Let $\boldsymbol{x}_{i,j}$ be the input vector of each data sample $j$ of device $i$, and $y_{i,j}$ be the output label in $\mathcal{X}_{s_i}$.
    Hence, the server 
    computes the loss as 
    \begin{equation} \label{eq:loss_server}
        \mathcal{L} = \sum_{k \in \mathcal{X}_{s_1}}l\left(\boldsymbol{z}_{S}, \boldsymbol{x}_{1, k}, y_{1, k}\right),
    \end{equation}
    where $l\left(\boldsymbol{z}_{S}, \boldsymbol{x}_{1, k}, y_{1, k}\right)$ is the loss value of $\boldsymbol{z}_{S}$ with respect to $\boldsymbol{x}_{1, k}$ and $y_{1, k}$.
    Given the loss value in (\ref{eq:loss_server}), each device or the server must compute the gradient $\frac{\partial \mathcal{L}}{\partial \boldsymbol{z}_k}$ with respect to the output $\boldsymbol{z}_{k-1}$ from device $s_{k}$, and the gradient $\frac{\partial \mathcal{L}}{\partial \boldsymbol{\theta}_k}$ with respect to the local parameter $\boldsymbol{\theta}_{k}$. These gradient vectors will be transmitted back to the selected devices in $\mathcal{S}$ for their model updates.
    The equation of calculating $\frac{\partial \mathcal{L}}{\partial \boldsymbol{\theta}_k}$ is
    \begin{equation} \label{eq:gradient_theta}
        \frac{\partial \mathcal{L}}{\partial \boldsymbol{\theta}_k} = \frac{\partial \mathcal{L}}{\partial \boldsymbol{z}_k} \times \nabla_{\boldsymbol{\theta}_k} f_k\left(\boldsymbol{z}_{k-1}; \boldsymbol{\theta}_k\right), k= 1,\ldots,S,
    \end{equation}
    where \( \nabla_{\boldsymbol{\theta}_k} f_k\left(\boldsymbol{z}_{k-1}; \boldsymbol{\theta}_k\right) \) is the gradient of the function \( f_k \) with respect to the parameter \( \boldsymbol{\theta}_k \). Next, each device $s_k$ or the server computes its gradient $\frac{\partial \mathcal{L}}{\partial \boldsymbol{z}_{k-1}}$, which is given by
    \begin{equation} \label{eq:gradient_send}
        \frac{\partial \mathcal{L}}{\partial \boldsymbol{z}_{k-1}}\!\! =\!\! \frac{\partial \mathcal{L}}{\partial \boldsymbol{z}_k}\!\! \times \!\! \nabla_{\boldsymbol{z}_{k-1}} f_k\left(\boldsymbol{z}_{k-1}; \boldsymbol{\theta}_k\right), k= 2,\ldots,S.
    \end{equation}
    $\frac{\partial \mathcal{L}}{\partial \boldsymbol{z}_{k-1}}$ will be transmitted to device $s_{k-1}$ to calculate $\frac{\partial \mathcal{L}}{\partial \boldsymbol{\boldsymbol{\theta}}_{k-1}}$ as done in (\ref{eq:gradient_theta}).
    Then, by using $\frac{\partial \mathcal{L}}{\partial \boldsymbol{\theta}_k}$, each device $s_k$ or the server updates its sub-model $\boldsymbol{\theta}_k$ based on a gradient descent method, as 
        $\boldsymbol{\theta}_k \!\leftarrow\! \boldsymbol{\theta}_k - \eta \frac{\partial \mathcal{L}}{\partial \boldsymbol{\theta}_k}, k=1,\ldots,S$,
    where \( \eta \) is the learning rate.
    
\subsection{Data Transmission Model}
In our considered network, the time
division multiple access (TDMA) protocol is used to transmit model training information or send deceptive signals \cite{anwar2015tdma}. The data rate of device $s_i$ sending information to device $s_j$ is
\begin{equation} \label{eq:data_rate}
    c_{s_i,s_j}\!\!=\!\!B\log_2\left(1+\frac{p_{s_i}^{\left(s_i,s_j\right)}h_{s_i,s_j}}{\sum_{d \in \mathcal{D}_{s_i,s_j}}p_{d}^{\left(s_i,s_j\right)}h_{d,s_j}+BN_0}\right),
\end{equation}
where $B$ is the bandwidth; $p_{s_i}^{\left(s_i,s_j\right)}$ is the transmit power of device $s_i$ when transmitting information to device $s_j$; $h_{s_i,s_j} = om_{s_i,s_j}^{-2}$ is the channel gain between device $s_i$ and $s_j$, with $o$ being the Rayleigh fading parameter, $m_{s_i,s_j}$ being the distance between device $s_i$ and $s_j$; $\mathcal{D}_{s_i,s_j}$ is a subset of devices in $\left(\mathcal{U} \setminus \left\{s_i,s_j\right\}\right)$ that will transmit deceptive signals when device $s_i$ transmits information to device $s_j$; $\sum_{d \in \mathcal{D}_{s_i,s_j}}p_{d}^{\left(s_i,s_j\right)}h_{d,s_j}$ is the interference introduced by the devices in $\mathcal{D}_{s_i,s_j}$ that transmit deceptive signals; $BN_0$ is the additive white Gaussian noise (AWGN), with $N_0$ representing the power spectral density \cite{liu2012additive}.

\subsection{Time Consumption Model}
The time that the devices use for model training consists of four parts: 1) the delay of model output transmission, 2) the delay of gradient vector transmission, 3) the delay on model output calculation, and 4) the delay on gradient update, which are specified as follows.
\subsubsection{Delay of model output transmission} Given (\ref{eq:data_rate}), the time that device $s_k$ to transmits its model $\boldsymbol{z}_k$ to device $s_{k+1}$ is
\begin{equation} \label{eq:trans_time_forward}
    T_{s_k,s_{k+1}}^{\textrm{S}}\left(\boldsymbol{\theta}_k, \boldsymbol{\theta}_{k+1}\right) = \frac{\Gamma\left(\boldsymbol{z}_k\right)}{c_{s_k, s_{k+1}}}, k=1,\ldots,S-1,
\end{equation}
where $\Gamma\left(\boldsymbol{z}_k\right)$ is the size of output $\boldsymbol{z}_k$. 
Since the sub-model output is transmitted from device $s_1$ to the server, $k\!=\!1,\ldots,S-1$.

\subsubsection{Delay of gradient vector transmission}
The time required for device $s_k$ to transmit $\frac{\partial \mathcal{L}}{\partial \boldsymbol{z}_{k-1}}$ to device $s_{k-1}$ is
\begin{equation} \label{eq:trans_time_back}
    T_{s_k,s_{k-1}}^{\textrm{G}}\left(\boldsymbol{\theta}_k, \boldsymbol{\theta}_{k-1}\right)= \frac{\Gamma\left(\frac{\partial \mathcal{L}}{\partial \boldsymbol{z}_{k-1}}\right)}{c_{s_k, s_{k-1}}}, k=2, \ldots, S.
\end{equation}
In (\ref{eq:trans_time_back}), $k=2,\ldots,S$ stems from the fact that the gradient vector is transmitted from the server to device $s_1$.

\subsubsection{Device computing model}
The time that each device $s_k$ calculates its output $\boldsymbol{z}_k$ is
\begin{equation} \label{eq:proc_time_forward}
    T_{s_k}^{\textrm{F}}\left(\boldsymbol{\theta}_k\right)=\frac{\omega^{\textrm{B}}\lambda_f^{\boldsymbol{\theta}_k}\Gamma\left(\boldsymbol{z}_k\right)\Gamma\left(\boldsymbol{\theta}_k\right)}{f^{\textrm{B}}}, k=1,\ldots,S,
\end{equation}
where $f^{\textrm{B}}$ is the frequency of the central processing unit (CPU) clock of each device $s_k$, which is assumed to be equal for all devices; $\omega^{\textrm{B}}$ is the number of CPU cycles required for computing data (per bit); $\lambda_f^{\boldsymbol{\theta}_k}$ is the coefficient determined by the complexity of the model $\boldsymbol{\theta}_k$ in model training.
Similarly, the time that device $s_k$ updates its model $\boldsymbol{\theta}_k$ is
\begin{equation} \label{eq:proc_time_back}
    T_{s_k}^{\textrm{B}}\left(\boldsymbol{\theta}_k\right)=\frac{\omega^{\textrm{B}}\lambda_b^{\boldsymbol{\theta}_k}\Gamma\left(\frac{\partial \mathcal{L}}{\partial \boldsymbol{z}_{k-1}}\right)\Gamma\left(\boldsymbol{\theta}_k\right)}{f^{\textrm{B}}}, k=1,\ldots,S,
\end{equation}
where $\lambda_b^{\boldsymbol{\theta}_k}$ is the coefficient related to the complexity of the model $\boldsymbol{\theta}_k$ in model updating.

Given \eqref{eq:proc_time_forward} and \eqref{eq:proc_time_back}, 
the total delay of updating the model per iteration is 
\begin{equation} \label{eq:tot_time}
\begin{aligned}
    &T_{\textrm{tot}}\left(\boldsymbol{\theta}_1,\ldots, \boldsymbol{\theta}_S \right) \\ &= \sum_{k=1}^{S-1}T_{s_k,s_{k+1}}^{\textrm{S}}\left(\boldsymbol{\theta}_k, \boldsymbol{\theta}_{k+1}\right) + \sum_{k=2}^{S}T_{s_k,s_{k-1}}^{\textrm{G}}\left(\boldsymbol{\theta}_k, \boldsymbol{\theta}_{k-1}\right) \\
    & \quad + \sum_{k=1}^{S}T_{s_k}^{\textrm{F}}\left(\boldsymbol{\theta}_k\right)+\sum_{k=1}^{S}T_{s_k}^{\textrm{B}}\left(\boldsymbol{\theta}_k\right),
\end{aligned}
\end{equation}
where $\sum_{k=1}^{S-1}T_{s_k,s_{k+1}}^S\left(\boldsymbol{\theta}_k, \boldsymbol{\theta}_{k+1}\right)$ is the delay of sub-model output transmitted from device $s_1$ to the server; $\sum_{k=2}^{S}T_{s_k,s_{k-1}}^{\textrm{G}}\left(\boldsymbol{\theta}_k, \boldsymbol{\theta}_{k-1}\right)$ is the delay of gradient vector transmitted from the server to device $s_1$; $\sum_{k=1}^{S}T_{s_k}^{\textrm{F}}\left(\boldsymbol{\theta}_k\right)$ is the consumption time of all devices in $\mathcal{D}$ that compute their models; $\sum_{k=1}^{S}T_{s_k}^{\textrm{B}}\left(\boldsymbol{\theta}_k\right)$ is the consumption time of these devices updating their gradients.

\subsection{Energy Consumption Model}
The energy consumption of model training consists of two components: a) data transmission and b) data computing. Then, the total energy consumption per iteration is
\begin{equation} \label{eq:energy_tot}
\begin{aligned}
    &E_{tot}\left(\boldsymbol{\theta}_1,\ldots,\boldsymbol{\theta}_S\right) \\ & = \sum_{k=1}^S\left(\vartheta_k\left(f^\textrm{B}\right)^2\left(\lambda_f^{\boldsymbol{\theta}_k}G\left(\boldsymbol{\theta}_k\right)+\lambda_b^{\boldsymbol{\theta}_k}G\left(\boldsymbol{\theta}_k\right) \right)\right) \\
    & \ +\!\! \sum_{k=1}^{S-1}\!\!\left(\!\!p_{s_k}^{\left(s_k,s_{k+1}\right)}\!\! +\! \!\!\!\!\!\!\!\!\!\sum_{d \in \mathcal{D}_{s_k,s_{k+1}}}\!\!\!\!\!p_d^{\left(s_k,s_{k+1}\right)}\!\!\right)\!\!T_{s_k, s_{k+1}}^{\textrm{S}}\!\!\!\left(\boldsymbol{\theta}_k, \boldsymbol{\theta}_{k+1}\right) \\
    & \ +\!\! \sum_{k=2}^{S}\!\!\left(\!\!p_{s_k}^{\left(s_k,s_{k-1}\right)} \!\!+ \!\!\!\!\!\!\!\!\!\sum_{d \in \mathcal{D}_{s_k,s_{k-1}}}\!\!\!\!\!p_d^{\left(s_k,s_{k-1}\right)}\!\!\right)\!\!T_{s_k,s_{k-1}}^{\textrm{G}}\!\!\!\left(\boldsymbol{\theta}_k, \boldsymbol{\theta}_{k-1}\right),
\end{aligned}
\end{equation}
where $\vartheta_k$ is the energy consumption coefficient depending on the chip of device $s_k$; $G\left(\boldsymbol{\theta}_k\right)$ is the size of model $\boldsymbol{\theta}_k$. In (\ref{eq:energy_tot}), the first term is the sum of the energy consumption of each device computing and updating its model; $\sum_{k=1}^{S-1}p_{s_k}^{\left(s_k,s_{k+1}\right)}T_{s_k,s_{k+1}}^{\textrm{S}}\left(\boldsymbol{\theta}_k,\boldsymbol{\theta}_{k+1}\right)$ represents the energy consumption of transmitting sub-model output from device $s_1$ to the server, and $ \sum_{k=2}^{S}p_{s_k}^{\left(s_k,s_{k-1}\right)}T_{s_k,s_k-1}^{\textrm{G}}\left(\boldsymbol{\theta}_k,\boldsymbol{\theta}_{k-1}\right)$ represents the energy consumption of transmitting the gradient vector from the server to device $s_1$; $\sum_{d \in \mathcal{D}_{s_i,s_{j}}}\!\!\!p_d^{\left(s_i,s_j\right)}T_{s_i,s_j}^{\textrm{S}}\!\!\left(\boldsymbol{\theta}_i,\boldsymbol{\theta}_j\right)$ is the energy consumption of the devices that transmit deceptive signals when device $s_i$ sends information to device $s_j$.

\subsection{Eavesdropper Model}
Each eavesdropper $e \in \mathcal{E}$ monitors the model information transmitted by the devices in $\mathcal{S}$. Each eavesdropper $e$ will receive several signals that are transmitted by one device in $\mathcal{S}$ that transmits actual model information and a subset of devices in $\mathcal{D}$ that transmit deceptive signals. Here, we assume that each eavesdropper will only consider the signal with highest SNR as the useful information \cite{10547445}. Then, when device $s_i$ sends information to device $s_j$, the signal of the device in $\mathcal{U}$ captured by eavesdropper $e$ is
\begin{equation} \label{eq:arg_snr}
    \epsilon_{s_i,s_j}(e) = \arg \max_{k \in \left(\left\{s_i\right\} \cup \mathcal{D}_{s_i,s_j}\right)} \frac{p_k h_{k,e}}{BN_0},
\end{equation}
where $p_k$ is the transmit power of device $k$. Here, the eavesdropper either captures the deceptive signals or model information (output or gradient vector) which depends on the SNR values of the signals.
We assume that the useful information that an eavesdropper infers from device $s_i$ that transmits $\boldsymbol{\theta}_i$ to $s_j$ is $\delta_{s_i,s_j}\left(\boldsymbol{\theta}_i\right)$. 
Let $q_e^{\left(s_i,s_j\right)}$ be the probability that eavesdropper $e$ monitors the model information transmitted by device $s_i$. 
Then, the amount of information leaked to all eavesdroppers per iteration is
\begin{equation} 
\begin{aligned}
    &\mathbb{E}\left[I\left(\boldsymbol{\theta}_1,\ldots,\boldsymbol{\theta}_S\right)\right]\\&=\sum_{e \in \mathcal{E}}\left(\sum_{k=1}^{S-1}\mathds{1}_{\left\{\epsilon_{s_k,s_{k+1}}\left(e\right) = s_k\right\}}q_e^{\left(s_k,s_{k+1}\right)}\delta_{s_k,s_{k+1}}\!\!\left(\boldsymbol{\theta}_k\right)\right.
    \\
    &\quad \left. + \sum_{k=2}^{S}\mathds{1}_{\left\{\epsilon_{s_k,s_{k-1}}\left(e\right) = s_k\right\}}q_e^{\left(s_k,s_{k-1}\right)}\delta_{s_k,s_{k-1}}\!\!\left(\boldsymbol{\theta}_k\right)\right),
\end{aligned}
\end{equation}
where $\mathds{1}_{\left\{\epsilon_{s_k,s_{k+1}}\left(e\right) = s_k\right\}}$ is the function that indicates whether eavesdropper $e$ recognizes its identity as the model training device $s_k$. Specifically, $\mathds{1}_{\left\{\epsilon_{s_k,s_{k+1}}\left(e\right) = s_k\right\}} = 1$ implies that the deceptive signals fail to confuse the eavesdropper, which means the eavesdropper still intercepts the training model information. Otherwise, we have $\mathds{1}_{\left\{\epsilon_{s_k,s_{k+1}}\left(e\right) = s_k\right\}} = 0$.

\subsection{Problem Formulation}
Given the defined models, our goal is to minimize the amount of information leaked to eavesdroppers per iteration while meeting time delay and energy consumption requirements of devices. This minimization problem involves optimizing the device selection strategy including determining the set $\mathcal{S}$ of model training devices, and the subset of $\left(\mathcal{U} \setminus \left\{s_i,s_j\right\}\right)$ of devices that transmit deceptive signals when each device $s_i$ transmits model information to device $s_j$, the model split strategy of $\boldsymbol{\Theta}_{\textrm{G}}$, and the transmit power strategy for selected devices in $\mathcal{U}$. The minimization problem is formulated as
\begin{equation} \label{eq:problem_formulation}
    \begin{aligned}
        \min_{\substack{\mathcal{S}, \boldsymbol{\Theta}_{\textrm{G}}, \boldsymbol{P},\\\left\{\mathcal{D}_{s_k, s_{k+1}}^{[1:S-1]}\right\},\\\left\{\mathcal{D}_{s_k, s_{k-1}}^{[2:S]}\right\}}}  \mathbb{E}\left[I\left(\boldsymbol{\theta}_1,\ldots,\boldsymbol{\theta}_S\right)\right],
    \end{aligned}
\end{equation}
\vspace{-0.7cm}

\begin{align} 
    \text{s.t.} \quad & \mathcal{S} \subseteq \mathcal{U}, \tag{14a} \label{eq:constraint_a} \\
    & \mathcal{D}_{s_i,s_j} \subseteq \mathcal{U} \setminus \left\{s_i,s_j\right\}, \forall i \neq j, \tag{14b} \label{eq:constraint_b} \\
    & \boldsymbol{\theta}_k \in \mathbb{R}^{n_k},\!\!\! \; \sum_{k=1}^S \!n_k = n_G, \tag{14c} \label{eq:constraint_c} \\
    &T_{tot}\left(\boldsymbol{\theta}_1,\ldots,\boldsymbol{\theta}_S\right) \leq \gamma_T, \tag{14d} \label{eq:constraint_d} \\
    &E_{tot}\left(\boldsymbol{\theta}_1,\ldots,\boldsymbol{\theta}_S\right) \leq \gamma_E, 
    \tag{14e} \label{eq:constraint_e}
\end{align}
where $\boldsymbol{P} = \left[p_1, \ldots, p_U \right]$, $n_k$ and $n_G$ are the dimensions of sub-model $\boldsymbol{\theta}_k$ and globe model $\boldsymbol{\Theta}_{\textrm{G}}$; $\gamma_T$ is the delay requirement of model training per iteration; $\gamma_E$ is the energy requirement of all the devices per iteration. Constraint (\ref{eq:constraint_a}) indicates that some device in $\mathcal{U}$ can be selected to train model. Constraint (\ref{eq:constraint_b}) shows that the devices sending deceptive signals are selected from $\mathcal{U}$ excluding the devices transmitting or receiving information in $\mathcal{S}$. Constraint (\ref{eq:constraint_c}) indicates that the global model is divided into $S$ separate sub-models and guarantees that the sub-models are assigned to the server and devices in $\mathcal{S}$. Constraint (\ref{eq:constraint_d}) is the delay requirement of model training per iteration. Constraint (\ref{eq:constraint_e}) indicates the maximum energy consumption of all the devices per iteration.

The problem in (\ref{eq:problem_formulation}) is challenging to solve due to the following reasons. First, the optimization variables (e.g., $\mathcal{S}$, $\mathcal{D}_{s_i,s_j}$, $\boldsymbol{\Theta}_{\textrm{G}}$, and $\boldsymbol{P}$) are coupled, since 1) each device can either be a model training device or a deceptive signal transmission device, and 2) the transmit power allocation depends on the model training devices selection, deceptive signal transmission devices selection, and the sub-model allocation. 
Second, neither the server nor the devices have prior knowledge of the eavesdroppers’ behavior or monitoring patterns (i.e., the monitoring probability $q_e^{(s_i,s_j)}$). Hence, traditional optimization algorithms without learning abilities \cite{liu2024survey} may not be able to solve this problem with unknown parameters. 
Third, since (\ref{eq:tot_time}) and (\ref{eq:energy_tot}) limit the energy and latency of the entire SL training across several training steps, the optimization variable decisions at one step will affect future optimization variable decisions, which further complicates the optimization 
problem.

\section{Proposed Soft Actor-Critic with Intrinsic Curiosity Module and Cross-Attention}

\begin{figure}[tp]
    \centering
\includegraphics[width=.4\textwidth]{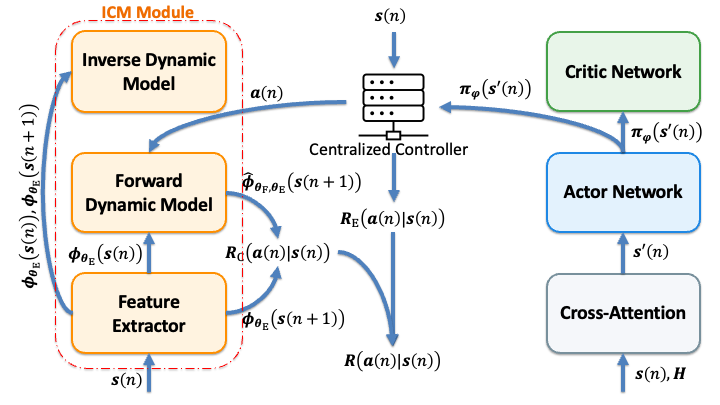}
\vspace{-0.2cm}
    \caption{The architecture of the ICM-CA.}
    \label{fig:flowchart}
      \vspace{-0.7cm}
\end{figure}

To address the problem in (\ref{eq:problem_formulation}), as shown in Fig. \ref{fig:flowchart}, we design a novel soft actor-critic deep reinforcement learning framework with intrinsic curiosity module and cross-attention (ICM-CA), which enables the server to optimize $\mathcal{S}$, $\mathcal{D}_{s_i,s_j}$, $\boldsymbol{\Theta}_{\textrm{G}}$, and $\boldsymbol{P}$ based on the observed eavesdroppers' activities, the remaining training model, time, and energy. Since the considered problem includes a large amount of optimization variables, standard RL methods \cite{nguyen2020deep, gao2024multi} may not be applied due to the large action and state spaces defined based on the number of optimization variables. Our proposed method leverages the ICM to encourage the server to explore novel actions and states, thus discovering better policies that can further reduce the amount of information leaked to eavesdroppers while meeting time and energy consumption constraints. 
To further improve the training efficiency, the proposed method uses a cross-attention mechanism to determine the importance of each historical state-action pair such that the designed RL method can focus more on the important state-action pairs \cite{10506764}.
Next, we first introduce the components of the ICM-CA framework and then detail its training process.

\subsection{Components of the ICM-CA Framework}

The fundamental components of the designed method are:

\textbf{1) Agent:} The agent is a centralized controller (e.g., the server) that determines the set of model training devices and their assigned sub-models, the set of deceptive signal transmission devices, and the transmit power of both model training and deceptive signal transmission devices, in order to reduce the amount of information leaked to eavesdroppers.

\textbf{2) State Space:} The state captures the current wireless network status and the SL training status. At each step $n$, the state is represented as
\begin{equation} \label{eq:state}
    \boldsymbol{s}\!\left(n\right)\!\! =\!\! \left[\!E_{\textrm{R}}\!\left(n\right)\!,\! T_{\textrm{R}}\!\left(n\right)\!,\! \boldsymbol{\Theta}_{G}'\!\!\left(n\right)\!, \!\boldsymbol{r}\!\left(n\right)\!,\! v\!\left(n\right)\!,\!\boldsymbol{l}_{\textrm{M}}\!\left(n\right)\!,\! \boldsymbol{l}_{\textrm{D}}\!\left(n\right)\!\right]\!, 
\end{equation}
where $E_{\textrm{R}}\left(n\right)$ and $T_{\textrm{R}}\left(n\right)$ are the remaining energy and time that the devices can use for model training and deceptive signal transmission, $\boldsymbol{\Theta}_{G}'\left(n\right)$ is the portion of the global model that are not assigned to the devices, and  $\boldsymbol{r}\left(n\right) = \left[r_1\left(n\right),r_2\left(n\right), \ldots,r_U\left(n\right)\right]$ is a vector of the model training device index where each index indicates if a device is assigned to train the sub-model. Specifically, $r_m\left(n\right)=0$ indicates that device $m$ is not selected to join SL training, while $r_m\left(n\right) = k$ implies that sub-model $k$ has been assigned to device $m$. $v\left(n\right)$ is the index of the model training device that transmits the output of the training model. $\boldsymbol{l}_{\textrm{M}}\left(n\right)$ and $\boldsymbol{l}_{\textrm{D}}\left(n\right)$ are the vectors of the distance between device $v\left(n\right)$ and the eavesdroppers, as well as other devices in $\mathcal{U}$, respectively. When $n > S$, all the sub-models and model training devices are assigned. Hence, $\boldsymbol{\Theta}_{G}'\left(n\right)$ and $\boldsymbol{r}\left(n\right)$ will not be changed from steps $S+1$ to $2S$ where the devices transmit gradient vectors for model update.

\textbf{3) Action Space:} The action at each step determines the model training device that will receive the training model output from the device selected in the previous step, the size of the sub-model assigned to the selected model training device, the deceptive signal transmission devices, and the transmit powers of the model training device and deceptive signal transmission devices. Hence, each action is expressed as
\begin{equation}
    \boldsymbol{a}\left(n\right)= \left[u\left(n\right), \boldsymbol{\theta}_n, \boldsymbol{d}\left(n\right),\boldsymbol{p}\left(n\right), x(n)\right],
\end{equation}
where $u\left(n\right)$ is the selected model training device index; $\boldsymbol{\theta}_n$ is the sub-model assigned to device $u\left(n\right)$; $\boldsymbol{d}\left(n\right)$ represents the deceptive signal transmission devices; $\boldsymbol{p}\left(n\right)$ is the transmit power of the model training device and deceptive signal transmission devices; $x\left(n\right) \in \left[0,S\right)$ is the index of the model training device that receives the model gradient vector from its connected device. Here, when $n \leq S$, no devices need to transmit gradient vectors for model update and hence $x(n) = 0$, otherwise, we have $x(n) \neq 0$. Meanwhile, when $n > S$, all the models have been assigned to the model training devices. Hence, $u(n)$ and $\boldsymbol{\theta}_n$ will always be $0$ and $\boldsymbol{0}$ respectively. In consequence, the constant values of $x(n)$, $u(n)$, and $\boldsymbol{\theta}_n$ will introduce several invalid actions. For example, when $n>S$, all the actions with $u(n)>0$ or $ \boldsymbol{\boldsymbol{\theta}} = \boldsymbol{0}$ are invalid. To ensure the agent only selects valid actions during the training process, an action mask algorithm is used \cite{vecchio2020mask}.

\textbf{4) Intrinsic Curiosity Module:} 
The ICM is used to encourage the agent to explore unknown state-action pairs so as to improve the speed of finding the action-state pairs thus improving the RL model training. This method is particularly useful for our defined RL solution since it consists of a large action and state space which are time consuming to explore by traditional RL methods. The ICM model consists of a feature extractor, a forward dynamic model, and an inverse dynamic model. The feature extractor is approximated by a multi-layer percepyrons (MLP) with residual blocks network, which extracts the state features  that are relevant to action selection. The feature extractor is defined as
\begin{equation} \label{eq:feature_extractor}
    {\boldsymbol{\phi}}_{\boldsymbol{\theta}_\text{E}}\left(\boldsymbol{s}\left(n\right)) = f_{\boldsymbol{\theta}_\text{E}}(\boldsymbol{s}\left(n\right)\right),
\end{equation}
where \( {\boldsymbol{\phi}}_{\boldsymbol{\theta}_\text{E}}(\boldsymbol{s}\left(n\right)) \) is the extracted features of state \( \boldsymbol{s}\left(n\right) \), and \( \boldsymbol{\theta}_\text{E} \) is a vector of the parameters of the feature extractor.
The forward dynamic model that consists of MLP, residual blocks, and gated recurrent units (GRU), is used to predict the feature representation of the next state given the current state $\boldsymbol{s}\left(n\right)$ and action $\boldsymbol{a}\left(n\right)$. This state prediction is used to evaluate how well the agent learns the relationship between the current action,  the current state, and the next state. A high prediction error implies that the agent does not learn the relationship well such that it is unfamiliar with the state and the actions. In consequence, the agent must explore this action and state more in subsequent steps. The next state predicted by the forward dynamic model is expressed as
\begin{equation} \label{eq:forward_model}
    \hat{{\boldsymbol{\phi}}}_{\boldsymbol{\theta}_{\textrm{F}}, \boldsymbol{\theta}_{\textrm{E}}}(\boldsymbol{s}(n+1)) = f_{\boldsymbol{\theta}_{\textrm{F}}}({\boldsymbol{\phi}}_{\boldsymbol{\theta}_\text{E}}(\boldsymbol{s}\left(n\right)), \boldsymbol{a}\left(n\right)),
\end{equation}
where $\boldsymbol{\theta}_{\textrm{F}}$ represents the parameters of the neural networks in the forward dynamic model. 
In (\ref{eq:feature_extractor}) and (\ref{eq:forward_model}), the feature extractor may extract the features that are not related to the prediction of the next state feature vector.
To address this problem, the inverse dynamic model that predicts the probability distribution over all possible actions based on ${\boldsymbol{\phi}}_{\boldsymbol{\theta}_{\textrm{E}}}(\boldsymbol{s}\left(n\right))$ and ${\boldsymbol{\phi}}_{\boldsymbol{\theta}_{\textrm{E}}}(\boldsymbol{s}(n+1))$ is used to verify whether the state features extracted by the feature extractor capture all the state transition information influenced by the action. 
The inverse dynamic model is approximated by a neural network with the same architecture of the forward dynamic model. The output of the inverse dynamic model is
\begin{equation} \label{eq:inverse_model}
    \hat{\boldsymbol{p}}_{\boldsymbol{\theta}_{\textrm{I}}, \boldsymbol{\theta}_{\textrm{E}}}\left(n\right) = N\left(f_{\boldsymbol{\theta}_{\textrm{I}}}({\boldsymbol{\phi}}_{\boldsymbol{\theta}_{\textrm{E}}}(\boldsymbol{s}\left(n\right)), {\boldsymbol{\phi}}_{\boldsymbol{\theta}_{\textrm{E}}}(\boldsymbol{s}(n+1)))\right),
\end{equation}
where $\hat{\boldsymbol{p}}_{\boldsymbol{\theta}_{\textrm{I}}, \boldsymbol{\theta}_{\textrm{E}}}\left(n\right)$ represents the predicted
probability distribution over the action space at step n. In $\hat{\boldsymbol{p}}_{\boldsymbol{\theta}_{\textrm{I}}, \boldsymbol{\theta}_{\textrm{E}}}\left(n\right)$, a higher probability $\hat{p}^i_{\boldsymbol{\theta}_{\textrm{I}}, \boldsymbol{\theta}_{\textrm{E}}}\left(n\right)$ implies that action $i$ has a higher probability that changes the state from ${\boldsymbol{\phi}}_{\boldsymbol{\theta}_{\textrm{E}}}(\boldsymbol{s}\left(n\right))$ to ${\boldsymbol{\phi}}_{\boldsymbol{\theta}_{\textrm{E}}}(\boldsymbol{s}\left(n+1\right))$; $\boldsymbol{\theta}_{\textrm{I}}$ are the parameters of the inverse dynamic model; $N(\cdot)$ is the softmax function. 

\textbf{5) Reward Function:} The reward function in ICM-CA must capture both the amount of information leaked to eavesdroppers and the agent’s motivation to explore unknown states. 
\subsubsection{The reward function for information leaked to eavesdroppers}
The reward that captures the amount of model information leaked to eavesdroppers is
\begin{equation} \label{eq:reward_transmit}
\begin{aligned}
&R_{\textrm{E}}\left(\boldsymbol{a}\left(n\right)|\boldsymbol{s}\left(n\right)\right)\\
&= \!\!
\begin{cases}
    0, \quad \quad \quad \quad \quad \quad \quad \quad \quad \quad \quad \quad \quad \quad \ \ \text{if} \ n = 1, \\
    -\!\!\sum\limits_{e \in \mathcal{E}}\!\!\mathds{1}_{\left\{\epsilon_{s_{n-1},s_{n}}\left(e\right) = s_{n-1}\right\}}\delta_{s_{n-1},s_{n}}Z_e\!\!\left(s_{n\!-\!1}\!,\!s_{n}\right)\\
    \quad - \omega_1\mathds{1}_{\left\{E_{\textrm{R}}\left(n\right)\leq 0\right\}}
    -\omega_2\mathds{1}_{\left\{T_{\textrm{R}}\left(n\right)\leq 0\right\}}, \ \ \ \text{if} \ 2 \leq n \leq S, \\
    -\!\!\sum\limits_{e \in \mathcal{E}}\!\!\mathds{1}_{\!\left\{\!\epsilon_{s_{2S\!-\!n\!+\!1},s_{2S\!-\!n}}\!\!\left(e\right)\! =\! s_{2S\!-\!n\!+\!1}\!\right\}}\!\delta_{s_{2S\!-\!n\!+\!1},s_{2S\!-\!n}}\!\!Z_e\!\!\left(\!s_{2S\!-\!n\!+\!1}\!,\! s_{2S\!-\!n}\!\right)\\
    \quad - \omega_1\mathds{1}_{\left\{E_{\textrm{R}}\left(n\right)\leq 0\right\}}
    -\omega_2\mathds{1}_{\left\{T_{\textrm{R}}\left(n\right)\leq 0\right\}}, \ \ \  \text{if} \ S < n \leq 2S\!-\!1,
\end{cases}
\end{aligned}
\end{equation}
where
$\mathds{1}_{\left\{E_{\textrm{R}}\left(n\right)\leq 0\right\}}$ and $\mathds{1}_{\left\{T_{\textrm{R}}\left(n\right)\leq 0\right\}}$ indicate whether the time and the energy used for SL model training and deceptive signal transmission meet the corresponding requirements in (\ref{eq:constraint_d}) and (\ref{eq:constraint_e}). Specifically, if they meet the requirements (i.e., the remaining energy and time are positive), $\mathds{1}_{\left\{E_{\textrm{R}}\left(n\right)\leq 0\right\}} = \mathds{1}_{\left\{T_{\textrm{R}}\left(n\right)\leq 0\right\}}=0$; otherwise, $\mathds{1}_{\left\{E_{\textrm{R}}\left(n\right)\leq 0\right\}} = \mathds{1}_{\left\{T_{\textrm{R}}\left(n\right)\leq 0\right\}}=1$; $\omega_1$ and $\omega_2$ are weight parameters that control the weight penalty in terms of the energy and the time;
$\mathds{1}_{\left\{\epsilon_{s_{n-1},s_{n}}\left(e\right) = s_{n-1}\right\}}\delta_{s_{n-1},s_{n}}Z_e\!\!\left(s_{n\!-\!1}\!,\!s_{n}\right)$ is the information leaked to eavesdropper $e$ during the model output transmission from device $\boldsymbol{\theta_1}$ to the server. The reason that $R_{\textrm{E}}\left(\boldsymbol{a}\left(0\right)|\boldsymbol{s}\left(0\right)\right) = 0$ is that the first step ($n=1$) only determines the device that will transmit the model output and the second step ($n=2$) determines the device that will receive the model output from the device determined in the first step.
$Z_e\left(s_{n-1},s_{n}\right)$ is the monitoring indicator and is expressed as
\begin{equation}
\begin{aligned}
    &Z_e\left(s_{n-1},s_{n}\right) \\&=
\begin{cases} 
1, & \!\!\!\!\text{if eavesdropper $e$  monitors the model transmission,} \\
0, & \!\!\!\!\text{otherwise.}
\end{cases}
\end{aligned}
\end{equation}
Here, the relationship between $Z_e\left(s_{n-1},s_{n}\right)$ and the probability $q_e^{\left(s_{n-1},s_{n}\right)}$ is $\mathbb{P}\left(Z_e\left(s_{n-1},s_{n}\right)=1\right)=q_e^{\left(s_{n-1},s_{n}\right)}$; $\mathds{1}_{\!\left\{\!\epsilon_{s_{2S\!-\!n\!+\!1},s_{2S\!-\!n}}\!\!\left(e\right) = s_{2S\!-\!n\!+\!1}\!\right\}}\!\delta_{s_{2S\!-\!n\!+\!1},s_{2S\!-\!n}}\!Z_e\!\left(\!s_{2S\!-\!n\!+\!1}\!,\! s_{2S\!-\!n}\!\right)$ is the amount of information leaked to eavesdropper $e$ during the process of transmitting gradient vectors from the server to model training devices.

\subsubsection{The reward function for unknown state exploration} The reward that captures the agent’s motivation to explore unknown states is
\begin{equation} \label{eq:reward_icm}
    R_{\textrm{C}}\!\!\left(\!\boldsymbol{s}\!\left(n\right), \boldsymbol{a}\!\left(n\right)\!\right) \!\!=\!\! \frac{1}{2}\! \left\| {\boldsymbol{\phi}}_{\boldsymbol{\theta}_{\textrm{E}}}\!\!\left(\boldsymbol{s}\left(n\!+\!1\right)\right) \!\!-\! \!\hat{\boldsymbol{\phi}}_{\boldsymbol{\theta}_{\textrm{F}}, \boldsymbol{\theta}_{\textrm{E}}}\!\!\left(\boldsymbol{s}\left(n\!+\!1\right)\right) \right\|_2^2\!,
\end{equation}
where $\|\cdot\|_2^2$ represents the squared Euclidean distance (L2 norm).
From (\ref{eq:reward_icm}), we see that, if the forward dynamic model cannot accurately predict the next state, it implies that the state $\boldsymbol{s}\left(n\right)$ is unknown and the reward $R_{\textrm{C}}\left(\boldsymbol{s}\left(n\right), \boldsymbol{a}\left(n\right)\right)$ is large. Hence, a large reward in (\ref{eq:reward_icm}) encourages the agent to explore unknown states, accelerating the efficiency of model training.

Given the definitions of two rewards, the total reward of the proposed RL method is
\begin{equation} \label{eq:reward_total}
    R\!\left(\boldsymbol{a}\left(n\right)|\boldsymbol{s}\left(n\right)\right)\! =\! R_{\textrm{E}}\!\left(\boldsymbol{a}\left(n\right)|\boldsymbol{s}\left(n\right)\right) \!+\! \zeta R_{\textrm{C}}\!\left(\boldsymbol{a}\left(n\right)|\boldsymbol{s}\left(n\right)\right),
\end{equation}
where $\zeta$ is a weight parameter.

\textbf{(6) Actor Network with Cross-Attention:}  
The actor network with a cross-attention mechanism is responsible for selecting actions based on the current state and past action-state pairs. Here, the cross-attention is used to generate a combined state $\boldsymbol{s}'(n)$ as follows:
\begin{equation} \label{eq:att_state}
\begin{aligned}
    \boldsymbol{s}'(n) &= X\left(\boldsymbol{s}\left(n\right), \boldsymbol{H}\right)
    = N \left(\frac{\boldsymbol{Q} \boldsymbol{K}^T}{\sqrt{C}}\right) \boldsymbol{V},
\end{aligned}
\end{equation}
where $\boldsymbol{H} = \left[(\boldsymbol{s}(n-1), \boldsymbol{a}(n-1));...; (\boldsymbol{s}(n-I), \boldsymbol{a}(n-I))\right]$ represents a vector including the last $I$ observed state-action pairs; \( \boldsymbol{Q} = \boldsymbol{W}_\textrm{Q} [\boldsymbol{s}(n); \boldsymbol{H}] \) is the query matrix derived from the current state and past state-action pairs; \( \boldsymbol{K} = \boldsymbol{W}_{\textrm{K}} \boldsymbol{H} \) and \( \boldsymbol{V} = \boldsymbol{W}_\textrm{V} \boldsymbol{H} \) represent the key and value matrices, which encode the past \( L \) state-action pairs in \( \boldsymbol{H} \). Here, \( \boldsymbol{W}_\textrm{Q}, \boldsymbol{W}_\textrm{K}, \boldsymbol{W}_\textrm{V} \) are learnable weight matrices, and \( C \) represents the dimensionality of the vectors in \( \boldsymbol{K} \) (i.e., the number of columns of \( \boldsymbol{K} \)).  
The actor network is approximated by an MLP with parameters \( \boldsymbol{\varphi} \). Its input is \( \boldsymbol{s}'(n) \) in (\ref{eq:att_state}) and its output is an action distribution vector $\boldsymbol{\pi}_{\boldsymbol{\varphi}}(\boldsymbol{s}'(n))$ that captures the likelihood of selecting each action.

\textbf{(7) Critic Network:}  
The critic network evaluates the state and provides a state value estimation for stable policy learning. The critic network takes the current state \( \boldsymbol{s}(n) \) as input and outputs the expected cumulative reward starting from \( \boldsymbol{s}(n) \) under the policy $\boldsymbol{\pi}_{\boldsymbol{\varphi}}(\boldsymbol{s}'(n))$. The critic network is approximated by a MLP neural network with parameters $\boldsymbol{\psi}$ and the critic network output is $V_{\boldsymbol{\psi}}^{\boldsymbol{\pi}_{\boldsymbol{\varphi}}}(\boldsymbol{s}(n))$.

\subsection{Training Process of the ICM-CA Method}

The training process of the ICM-CA method consists of the following steps.

\textbf{1) Data Collection:}
The agent observes eavesdroppers' activities, the remaining training model, time, and energy to collect experiences where each experience is represented by $(\boldsymbol{s}\left(n\right), \boldsymbol{a}\left(n\right), R\left(\boldsymbol{a}(n)|\boldsymbol{s}(n)\right), \boldsymbol{s}(n+1))$ and stores them in a replay buffer.

\textbf{2) ICM Update:}  
The ICM module update includes inverse dynamic model update, forward dynamic model update, and feature extractor update, which are specified as follows.


\textbf{(i) Inverse Dynamic Model Update:}  
The loss function used to update the inverse dynamic model is 
\begin{equation} \label{eq:inv_loss}
    L_{\textrm{I}}\left(\boldsymbol{\theta}_{\textrm{I}}, \boldsymbol{\theta}_{\textrm{E}}\right) = - \boldsymbol{b}(n)^T \log \hat{\boldsymbol{p}}_{\boldsymbol{\theta}_{\textrm{I}}, \boldsymbol{\theta}_{\textrm{E}}}\left(n\right),
\end{equation}
where \( \boldsymbol{b}(n)=\begin{bmatrix}
        b_1(n),
        b_2(n), 
        \ldots, 
        b_{|\mathcal{A}|}(n)
    \end{bmatrix}^T \) is a one-hot encoding vector that represents the action \( \boldsymbol{a}(n) \) selected at step \( n \)
with $\mathcal{A}$ being the set of all the actions and $|\mathcal{A}|$ is the number of actions in $\mathcal{A}$. In particular $b_i(n) = 1$ implies that action $i$ in $\mathcal{A}$ is selected by the agent; otherwise, $b_i(n) = 0$.
The inverse dynamic model parameters \( \boldsymbol{\theta}_{\textrm{I}} \) and the feature extractor \( \boldsymbol{\theta}_{\textrm{E}} \) are updated as
    $\boldsymbol{\theta}_{\textrm{I}} \leftarrow \boldsymbol{\theta}_{\textrm{I}} - \eta_1 \nabla_{\boldsymbol{\theta}_{\textrm{I}}} L_{\textrm{I}}\left(\boldsymbol{\theta}_{\textrm{I}}, \boldsymbol{\theta}_{\textrm{E}}\right)$
and
    $\boldsymbol{\theta}_{\textrm{E}} \leftarrow \boldsymbol{\theta}_{\textrm{E}} - \eta_1 \nabla_{\boldsymbol{\theta}_{\textrm{E}}} L_{\textrm{I}}\left(\boldsymbol{\theta}_{\textrm{I}}, \boldsymbol{\theta}_{\textrm{E}}\right)$,
where $\eta_1$ is the learning rate.

\textbf{(ii) Forward Dynamic Model Update:}  
Once the inverse dynamic model has been updated, the forward dynamic model updates its parameters to predict the feature representation of the next state. The loss function is
\begin{equation} \label{eq:fw_loss}
    L_{\textrm{F}}\!\left(\boldsymbol{\theta}_{\textrm{F}}, \boldsymbol{\theta}_{\textrm{E}}\right) \!=\! \frac{1}{2}\! \left\|\!\hat{{\boldsymbol{\phi}}}_{\boldsymbol{\theta}_{\textrm{F}}, \boldsymbol{\theta}_{\textrm{E}}}(\boldsymbol{s}(n+1))\! -\! {\boldsymbol{\phi}}_{\boldsymbol{\theta}_{\textrm{E}}}(\boldsymbol{s}(n+1))\!\right\|_2^2.
\end{equation}
The forward dynamic model parameters \( \boldsymbol{\theta}_{\textrm{F}} \) are updated as
    $\boldsymbol{\theta}_{\textrm{F}} \leftarrow \boldsymbol{\theta}_{\textrm{F}} - \eta_2 \nabla_{\boldsymbol{\theta}_{\textrm{F}}} L_{\textrm{F}}\left(\boldsymbol{\theta}_{\textrm{F}}, \boldsymbol{\theta}_{\textrm{E}}\right)$.

\textbf{(iii) Feature Extractor Update:}  
The feature extractor must be updated using a loss function that considers both the performance of both inverse dynamic model and forward dynamic model. The loss function is expressed as
\begin{equation} \label{eq:emb_loss}
    L_{\textrm{E}}\left(\boldsymbol{\theta}_{\textrm{E}}\right) = L_{\textrm{F}}\left(\boldsymbol{\theta}_{\textrm{F}}, \boldsymbol{\theta}_{\textrm{E}}\right) + \upsilon L_{\textrm{I}}\left(\boldsymbol{\theta}_{\textrm{I}}, \boldsymbol{\theta}_{\textrm{E}}\right),
\end{equation}
where \( \upsilon \) is the weight parameter. Then, the feature extractor parameters \( \boldsymbol{\theta}_{\textrm{E}} \) are updated as
    $\boldsymbol{\theta}_{\textrm{E}} \leftarrow \boldsymbol{\theta}_{\textrm{E}} - \eta_3 \nabla_{\boldsymbol{\theta}_{\textrm{E}}} L_{\textrm{E}}\left(\boldsymbol{\theta}_{\textrm{E}}\right)$.

\textbf{3) Critic Update:}  
The critic network is updated by reducing the difference between its predicted value of the current state and a reference value based on future rewards. The critic loss function is given by
\begin{equation} \label{eq:critic}
    L_{\textrm{C}}\!({\boldsymbol{\psi}})\!\! =\!\! \mathbb{E} \!\!\left[ \!(R\!\left(\boldsymbol{a}(n)|\boldsymbol{s}(n)\right)\! +\! \gamma V_{\boldsymbol{\psi}}^{\boldsymbol{\pi}_{\boldsymbol{\varphi}}}\!(\boldsymbol{s}(n+1))\!\! -\!\! V_{\boldsymbol{\psi}}^{\boldsymbol{\pi}_{\boldsymbol{\varphi}}}\!(\boldsymbol{s}\left(n\right)))^2\! \right]\!.
\end{equation}
Then, the parameters \( {\boldsymbol{\psi}} \) of the critic network are updated using gradient descent as
    ${\boldsymbol{\psi}} \leftarrow {\boldsymbol{\psi}} - \eta_c \nabla_{{\boldsymbol{\psi}}} L_{C}({\boldsymbol{\psi}})$,
where \( \eta_c \) is the learning rate.

\textbf{4) Actor Update:}  
The loss function used to update the actor network is
\begin{equation} \label{eq:actor}
    L_{\textrm{A}}({\boldsymbol{\varphi}})\! =\! -\!\mathbb{E}\! \left[ \log \pi_{\boldsymbol{\varphi}}\!(\boldsymbol{a}\left(n\right) | \boldsymbol{s}'\left(n\right)) Y_{\boldsymbol{\psi}}\!\left(n\right)\! -\! \alpha H\!\left(\boldsymbol{\pi}_{\boldsymbol{\varphi}}\!\!\left(s'(n)\right)\right) \right],
\end{equation}
where $\pi_{\boldsymbol{\varphi}}(\boldsymbol{a}\left(n\right) | \boldsymbol{s}'\left(n\right))$ is the probability of selecting action $\boldsymbol{a}(n)$ by the agent based on action distribution $\boldsymbol{\pi}_{\boldsymbol{\varphi}}\left(s'(n)\right)$; \( \alpha \) is the parameter that controls the balance between maximizing rewards and maintaining diversity in action selection; $Y_{\boldsymbol{\psi}}(\boldsymbol{a}(n), \boldsymbol{s}(n)) = R\left(\boldsymbol{a}(n)|\boldsymbol{s}(n)\right) + \gamma V_{\boldsymbol{\psi}}^{\boldsymbol{\pi}_{\boldsymbol{\varphi}}}(\boldsymbol{s}(n+1)) - V_{\boldsymbol{\psi}}^{\boldsymbol{\pi}_{\boldsymbol{\varphi}}}(\boldsymbol{s}\left(n\right))$ measures whether taking action $\boldsymbol{a}\left(n\right)$, which results in the change of the state from $\boldsymbol{s}(n)$ to $\boldsymbol{s}(n+1)$, can achieve a higher accumulated reward compare to the value of the current state $\boldsymbol{s}(n)$; $Y_{\boldsymbol{\psi}}(\boldsymbol{a}(n), \boldsymbol{s}(n) > 0$ implies that taking action $\boldsymbol{a}(n)$ leads to a higher estimated accumulated reward than staying in the current state $\boldsymbol{s}(n)$; otherwise, we have $Y_{\boldsymbol{\psi}}(\boldsymbol{a}(n), \boldsymbol{s}(n) \leq 0$; 
\( H\!\left(\boldsymbol{\pi}_{\boldsymbol{\varphi}}\left(s'(n)\right)\right) \) represents the entropy of the distribution $\boldsymbol{\pi}_{\boldsymbol{\varphi}}\left(s'(n)\right)$, which measures the uncertainty in action selection. 
The actor network parameters \( {\boldsymbol{\varphi}} \) are updated as
    ${\boldsymbol{\varphi}} \leftarrow {\boldsymbol{\varphi}} - \eta_a \nabla_{{\boldsymbol{\varphi}}} L_{\textrm{A}}({\boldsymbol{\varphi}})$,
where \( \eta_a \) is the learning rate. The detailed training process of ICM-CA is summarized in Algorithm \ref{alg:sac_icm_ca}.

\subsection{Convergence and Implementation Analysis}
Next, we analyze the convergence and implementation of the proposed ICM-CA method.

\subsubsection{Convergence analysis}
Here, we analyze the convergence of the proposed ICM-CA method. Since the proposed RL framework consists of a critic network $\boldsymbol{\psi}$ and an actor network $\boldsymbol{\varphi}$, we must analyze the convergence for both of them. Let \( V^{\pi_{\boldsymbol{\varphi}}}(\boldsymbol{s}(n)) \) be the actual value function, i.e., the expected cumulative reward starting from state \( \boldsymbol{s}(n) \) under policy \( \pi_{\boldsymbol{\varphi}} \). Then, the convergence of the critic network $\boldsymbol{\psi}$ is summarized in the following lemma.

\begin{lemma}
\label{lem:critic_converge}
If (i) the reward \( R(\boldsymbol{a}(n) | \boldsymbol{s}(n)) \) is bounded and (ii) the size of action space $\left|\mathcal{A}\right|$ is finite (i.e., $\left|\mathcal{A}\right| < \infty$), then the output of the critic network will converge to the actual value function \( V^{\pi_{\boldsymbol{\varphi}}}(\boldsymbol{s}(n)) \) \cite{haarnoja2018soft}.
\end{lemma}
\begin{proof}
In our designed algorithm, the reward $R_{\textrm{E}}\left(\boldsymbol{a}\left(n\right)|\boldsymbol{s}\left(n\right)\right)$ in (\ref{eq:reward_transmit}) is bounded since its maximum value is $0$; Its minimum value is $-\left|\mathcal{E}\right|\delta_{s_{n-1},s_{n}}-\omega_1-\omega_2$ if $n \leq S$ and $\left|\mathcal{E}\right|\delta_{s_{2S-n+1},s_{2S-n}}-\omega_1-\omega_2$ if $S <n \leq 2S - 1$. Meanwhile, since each element in 
${\boldsymbol{\phi}}_{\boldsymbol{\theta}_{\textrm{E}}}\!\!\left(\boldsymbol{s}\left(n\!+\!1\right)\right) $ or $ \hat{\boldsymbol{\phi}}_{\boldsymbol{\theta}_{\textrm{F}}, \boldsymbol{\theta}_{\textrm{E}}}\!\!\left(\boldsymbol{s}\left(n\!+\!1\right)\right)$ is between $0$ and $1$ \cite{pathak2017curiosity}, $R_{\textrm{C}}\left(\boldsymbol{a}\left(n\right)|\boldsymbol{s}\left(n\right)\right)$ in (\ref{eq:reward_icm}) is also bounded. Therefore, the reward function \( R(\boldsymbol{a}(n) | \boldsymbol{s}(n)) \) in (\ref{eq:reward_total}) is bounded and hence our designed RL meets the first condition. Next, we prove that the action space $\mathcal{A}$ is finite. In an action $\boldsymbol{a}(n)$, the device selection $\mathcal{S}$ and $\mathcal{D}$ is limited by the total number of devices in the system, the split model $\boldsymbol{\theta}_k$ depends on the total number of layers in the ML model, and the transmit power for each device is chosen from a predefined set of discrete levels. Therefore, the action space is also finite. Hence, both conditions in Lemma~\ref{lem:critic_converge} are satisfied. This completes the proof.
\end{proof}

\noindent When the critic network converges, the actor network is guaranteed to converge, as proved in ~\cite[Theorem~1]{haarnoja2018soft}.

\subsubsection{Implementation analysis}
The proposed ICM-CA method includes a training stage and a decision-making stage. In the training stage, the selected model training device $s_k$ receives the output $\boldsymbol{z}_{k-1}$ from the previous device $s_{k-1}$. Then, it generates the output $\boldsymbol{z}_k$ and sends it to the next device $s_{k+1}$. During backpropagation, each device $s_k$ receives the loss gradient $\frac{\partial \mathcal{L}}{\partial \boldsymbol{z}_k}$ from device $s_{k+1}$, updates its local model $\boldsymbol{\theta}_k$, as well as computes and sends $\frac{\partial \mathcal{L}}{\partial \boldsymbol{z}_{k-1}}$ to device $s_{k-1}$. To calculate the policy $\boldsymbol{\pi}_{\boldsymbol{\varphi}}$, the server needs to collect the time and energy consumption of each device computing and transmitting the model information. In the decision-making stage, the trained actor network outputs a policy $\boldsymbol{\pi}_{\boldsymbol{\varphi}}$ to enable the server to determine the model training and deceptive signal transmission devices, model assignments, and transmit powers.

\begin{algorithm}[t]
\footnotesize
\caption{The ICM-CA Algorithm for Private SL}
\label{alg:sac_icm_ca}
\begin{algorithmic}[1]
    \STATE \textbf{Initialize:} Actor network $\boldsymbol{\pi}_{\boldsymbol{\varphi}}$, critic network $V_{\boldsymbol{\psi}}$, and replay buffer.
    
    \FOR{each episode}
        \FOR{each step $n = 1, 2, \dots, S$}
            \STATE \textbf{State Representation:} Compute the cross-attention-enhanced state representation $\boldsymbol{s}'\left(n\right)$ using (\ref{eq:att_state}).
            \STATE \textbf{Action Selection:} Sample action $\boldsymbol{a}\left(n\right) \sim \boldsymbol{\pi}_{\boldsymbol{\varphi}}(\boldsymbol{s}'\left(n\right))$.
            \STATE \textbf{Environment Interaction:} Execute action $\boldsymbol{a}\left(n\right)$ and observe next state $\boldsymbol{s}(n+1)$ and reward $R\left(\boldsymbol{a}(n)|\boldsymbol{s}(n)\right)$ using (\ref{eq:reward_total}).
            \STATE \textbf{Replay Buffer Update:} Observe and store each experience $(\boldsymbol{s}\left(n\right), \boldsymbol{a}\left(n\right), R\left(\boldsymbol{a}(n)|\boldsymbol{s}'(n)\right), \boldsymbol{s}(n+1))$ in replay buffer.
            
            \STATE \textbf{ICM Update:}
            \STATE \hspace{0.5cm} 1. Update inverse dynamic model and feature extractor.
            \STATE \hspace{0.5cm} 2. Update forward dynamic model.
            \STATE \hspace{0.5cm} 3. Update feature extractor with combined loss.
            
            \STATE \textbf{Critic and Actor Update:} 
            \STATE \hspace{0.5cm} 1. Update the critic network.
            \STATE \hspace{0.5cm} 2. Update the actor network.
            
        \ENDFOR
    \ENDFOR
\end{algorithmic}
\end{algorithm}


\section{Performance Analysis}
Next, we analyze how the transmit powers of the model training and deceptive signal transmission devices affect the information leaked to the eavesdroppers. The analytical results can be used to simply the RL action space. It can also find the optimal transmit powers of the model training and deceptive signal transmission devices to minimize the expected amount of information leaked to eavesdroppers under two special network settings.
\begin{theorem}\label{the:1}
Given a model training device $s_k$ that transmits the output of model $\boldsymbol{\theta}_k$ to model training device $s_{k+1}$, the set $\mathcal{D}_{s_k, s_{k+1}}$ of deceptive signal transmission devices, and the probability $q_e^{(s_k,s_{k+1})}$ that eavesdropper $e \in \mathcal{E}$ monitors the model information transmitted by $s_k$, the expectation of the amount of information leaked to all eavesdroppers at this model output transmission is  
\begin{equation}
\begin{aligned}
    &\mathbb{E}\left(I_{s_k,s_{k+1}}\right) \\
    &=\!\!\!\sum_{e \in \mathcal{E}}\!\prod_{d \in \mathcal{D}_{s_k,s_{k+1}}}\!\!\!\!\!\frac{p_{s_k}m_{s_k,e}^{-2}}{p_{d}m_{d,e}^{-2}\!+\!p_{s_k}m_{s_k,e}^{-2}}q_e^{s_k,s_{k+1}} \delta_e^{s_k,s_{k+1}}\!\!\left(\boldsymbol{\theta}_k\right).
\end{aligned}
\end{equation}
\end{theorem}
\begin{proof}
See Appendix A.
\end{proof}
\noindent From Theorem \ref{the:1}, we see that $\mathbb{E}\left(I_{s_k,s_{k+1}}\right)$ depends on the locations and transmit powers of model training and deceptive signal transmission devices, the eavesdropping probability, and the size of the training model. 
From Theorem \ref{the:1}, we see that, when $p_{s_k} = 0$, $\mathbb{E}\left(I_{s_k,s_{k+1}}\right) = 0$. This is because no model information is transmitted between model training devices $s_k$ and $s_{k+1}$. When $p_{d} = + \infty$, we also have $\mathbb{E}\left(I_{s_k,s_{k+1}}\right) = 0$. This is due to the fact that the strength of deceptive signals is large enough to confuse eavesdroppers.
Using Theorem \ref{the:1}, we can obtain the optimal transmit powers $p_{s_k}^*$ and $p_d^*, \forall d \in \mathcal{D}_{s_k,s_{k+1}}$ under the following two network scenarios: 1) a network has one deceptive signal device (i.e., $\left|\mathcal{D}_{s_k,s_{k+1}}\right| = 1$) and 2) a network has one eavesdropper (i.e., $\left|\mathcal{E}\right| = 1$) with the assumption that the interference among the deceptive signals is ignored. These optimal powers are analyzed in the following corollary.
\begin{corollary} \label{col:1}
    Consider a network where a model training device $s_k$ that transmits the output of model $\boldsymbol{\theta}_k$ to model training device $s_{k+1}$ with 1) the time constraint $B_{\textrm{T}}$ and the energy constraint $B_{\textrm{E}}$, 2) one deceptive signal transmission device (i.e., $\left|\mathcal{D}_{s_k,s_{k+1}}\right|=1$), and 3) $\left|\mathcal{E}\right|$ eavesdroppers, the optimal transmit powers of model training and the deceptive signal transmission device that can minimize $\mathbb{E}\left(I_{s_k,s_{k+1}}\right)$ are
    \begin{equation} \label{eq:col1_final}
    \begin{aligned}
        p_{s_k}^* &= \frac{BN_0 + {m^{-2}_{s_k,d}o}\frac{B_{\textrm{E}}}{B_{\textrm{T}}}}{m^{-2}_{s_k,s_{k+1}}o \, / \left(2^{\frac{\Gamma\left(\boldsymbol{z}_k\right)}{B_{\textrm{T}}B}}-1\right) + {m^{-2}_{s_k,d}o}}, \\
        p_d^* &= \frac{m^{-2}_{s_k,s_{k+1}}o\frac{B_{\textrm{E}}}{B_{\textrm{T}}} - BN_0\left(2^{\frac{\Gamma\left(\boldsymbol{z}_k\right)}{B_{\textrm{T}}B}}-1\right)}{m^{-2}_{s_k,s_{k+1}}o + {m^{-2}_{s_k,d}o}{\left(2^{\frac{\Gamma\left(\boldsymbol{z}_k\right)}{B_{\textrm{T}}B}}-1\right)}},
    \end{aligned}
    \end{equation}
\end{corollary}
\begin{proof}
See Appendix B.
\end{proof}
From Corollary \ref{col:1}, we see that, to guarantee $p_d^* > 0$, $B_{\textrm{T}}$ and $B_{\textrm{E}}$ must satisfy $m^{-2}_{s_k,s_{k+1}}o\frac{B_{\textrm{E}}}{B_{\textrm{T}}} - BN_0\left(2^{\frac{\Gamma\left(\boldsymbol{z}_k\right)}{B_{\textrm{T}}B}}-1\right) > 0$. For example, if $B_{\textrm{E}} = 0$, $m^{-2}_{s_k,s_{k+1}}o\frac{B_{\textrm{E}}}{B_{\textrm{T}}} - BN_0\left(2^{\frac{\Gamma\left(\boldsymbol{z}_k\right)}{B_{\textrm{T}}B}}-1\right) < 0$. Hence, we cannot find a valid $p_d^*$ to satisfy time and energy constraints (\ref{eq:col_1_mini_ini_c1}) and (\ref{eq:col_1_mini_ini_c2}). 

The optimal powers of model training and deceptive signal transmission devices under the network with one eavesdropper are given in the following corollary.
\begin{corollary} \label{col:2}
    Consider a scenario where a model training device $s_k$ that transmits the output of model $\boldsymbol{\theta}_k$ to model training device $s_{k+1}$ with 1) the time constraint $B_{\textrm{T}}$ and the energy constraint $B_{\textrm{E}}$, 2) the ignorance of the interference from deceptive signal transmission devices, 3) $\left|\mathcal{D}_{s_k,s_{k+1}}\right|$ deceptive signal transmission devices, and 4) one eavesdropper (i.e., $\left|\mathcal{E}\right| = 1$), the optimal transmit powers of model training and deceptive signal transmission devices that can minimize $\mathbb{E}\left(I_{s_k,s_{k+1}}\right)$ are
\begin{equation} \label{eq:col2_final}
\begin{aligned}
    p_{s_k}^* &= \frac{BN_0\left(2^{\frac{\Gamma\left(\boldsymbol{z}_k\right)}{B_{\textrm{T}}B}}-1\right)}{m^{-2}_{s_k,s_{k+1}}o},\\
    p_d^* &= \frac{m^{-2}_{s_k,s_{k+1}}o\frac{B_{\textrm{E}}}{B_{\textrm{T}}}-BN_0\left(2^{\frac{\Gamma\left(\boldsymbol{z}_k\right)}{B_{\textrm{T}}B}}-1\right)}{m^{-2}_{s_k,s_{k+1}}om_{d,e}^{-2}\sum_{d \in \mathcal{D}_{s_k, s_{k+1}}}m_{d,e}^2}, \forall d \in \mathcal{D}_{s_k,s_{k+1}},
\end{aligned}
\end{equation}
\end{corollary}
\begin{proof}
    See Appendix C.
\end{proof}
From Corollary \ref{col:2}, we see that, $p_{s_k}^*$ only relies on $B_{\textrm{T}}$ and will decrease if $B_{\textrm{T}}$ increases, and $p_d^*$ will increase if $B_{\textrm{E}}$ increases. This is because a longer transmission time allows the model training device to use lower power, and a larger energy constraint enables deceptive devices to allocate more power to generate interference, thus reducing the amount of information leaked to eavesdroppers while ensuring transmission efficiency.
For other scenarios, the closed-form solutions are not available.

\section{Simulation Results and Analysis}

\begin{table}[t]
\caption{Simulation Parameters}
\label{tb:simu_para}
\centering
\footnotesize
\setlength{\tabcolsep}{0.6mm}{
\begin{tabular}{|c|c||c|c|}
\hline
\textbf{Parameters}         & \textbf{Value}  & \textbf{Parameters}         & \textbf{Value}  \\ \hline
$S$ & $4$ & $
\lambda_b^{(\boldsymbol{\theta}_k)}$ (FLOPs)& $1 \text{-} 2 (\times 10^9)$ \\ \hline
$\omega^{\textrm{B}}$ (cycles/bit) & $10^4\text{-}10^6$ & $o$ & 1 \\ \hline
$f^{\textrm{B}}$ (GHz) & $4\text{-}7$ & $\gamma_T$ (s) & 8 \\ \hline
$B$ (MHz) & $1$ & $\gamma_E$ (J) &75 \\ \hline
$N_0$ (dBm/Hz) & $-90$ & $q_e^{\left(s_i,s_j\right)}$ & 0.8 \\ \hline
$\zeta$ & $0.3$ & $\eta_a$ & $1 \times 10^{-4}$ \\ \hline
$\upsilon$ & $5\text{-}8$ & $\eta_c$ & $3 \times 10^{-4}$ \\ \hline
\end{tabular}
}
\vspace{-0.4cm}
\end{table}

For our simulations, we consider a network setup consisting of $U = 6$ devices and $E = 2$ eavesdroppers, distributed randomly within a $800 \times 800\, m^2$ area. 
We use the ImageNet-1k dataset to train a ResNet-101 model \cite{he2016deep}. Other parameters are listed in Table \ref{tb:simu_para}.
For comparison purposes, we consider four baselines:  
a) the proposed method without the ICM module, where the agent is guided only by the reward $R_{\textrm{E}}(\boldsymbol{a}(n)|\boldsymbol{s}(n))$ produced by the RL framework;  
b) the proposed method without the CA module, where the actor network makes decisions only based on the current state $\boldsymbol{s}(n)$ without incorporating historical state-action pairs.  
Baselines a) and b) are used to evaluate how the ICM and CA modules improve the overall performance of our method. 
c) the proximal policy optimization (PPO) method, which updates the policy using a clipped objective function to avoid large and unstable policy shifts  \cite{li2023uav};  
d) the Q-learning algorithm, which learns a value function to estimate the expected return of each action in a discrete action space \cite{zhang2017energy}.  
Baselines c) and d) are used to compare our proposed SAC based framework with other commonly used RL algorithms.

\begin{figure}[tp]
    \centering
\includegraphics[width=.45\textwidth]{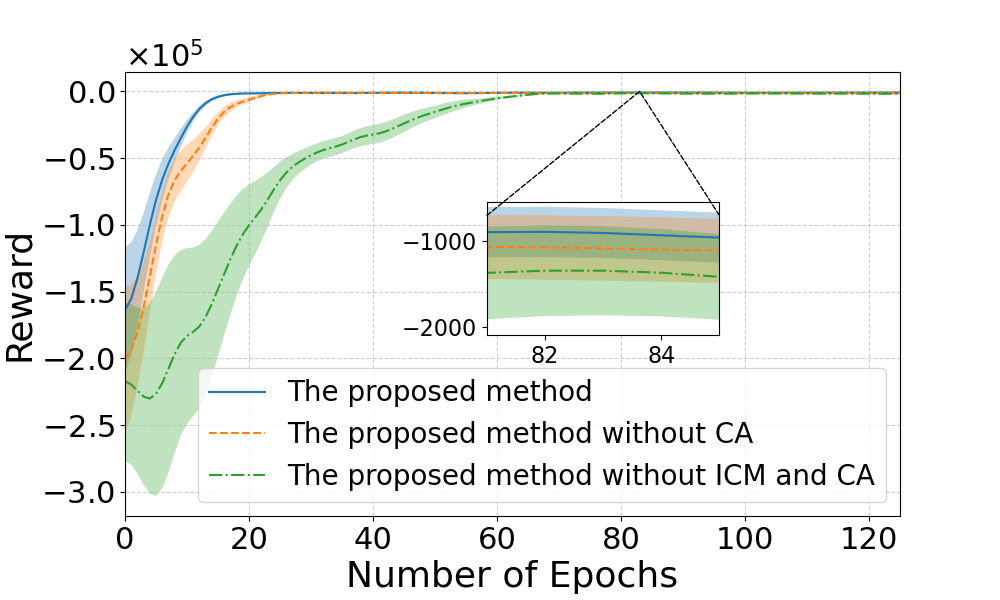}
\vspace{-0.2cm}
    \caption{Convergence of the proposed method.}
    \label{fig:f3}
      \vspace{-0.5cm}
\end{figure}

In Fig. \ref{fig:f3}, we show the convergence of our proposed ICM-CA method and evaluate how ICM and CA module affects the convergence of the proposed algorithm. From this figure, we see that the proposed method improves the convergence rate and the accumulated reward by up to $3 \times$ and $30 \%$ compared to the proposed method without ICM. This is because the ICM encourages the server to explore more actions and states during the training process. We also see that, with the CA module, the proposed method improves the accumulated reward by up to $9 \%$. This is due to the fact that the CA module can find the important state-action pairs to improve the training of the actor network.

\begin{figure}[tp]
    \centering
\includegraphics[width=.45\textwidth]{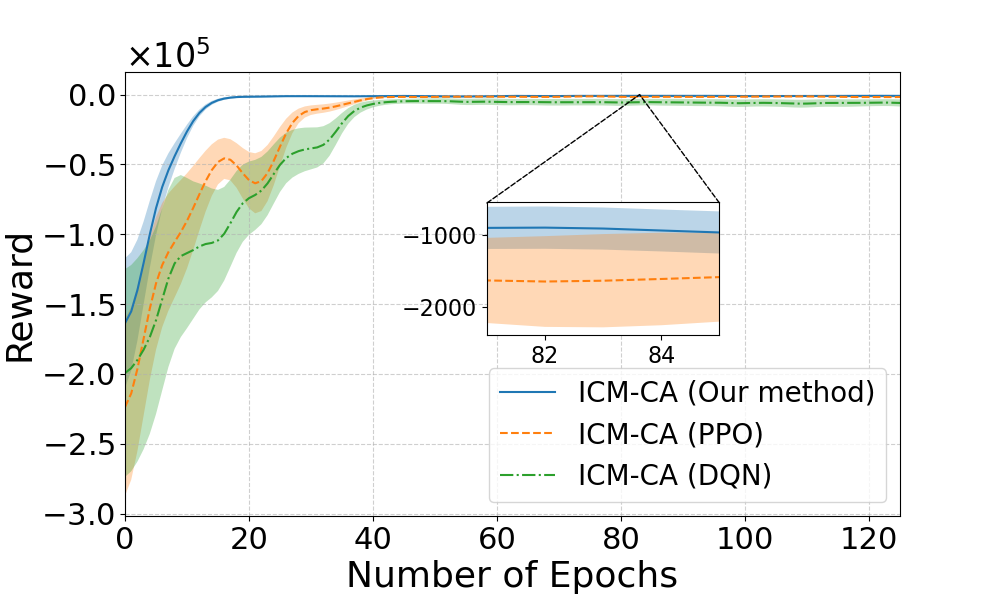}
\vspace{-0.2cm}
    \caption{The convergence of the proposed method with different neural network models.}
    \label{fig:f4}
      \vspace{-0.5cm}
\end{figure}

In Fig. \ref{fig:f4}, we show the convergence of the proposed method with different neural network models. From this figure, we see that our method improves the convergence rate by up to $2 \times$ compared to those using PPO and DQN. Fig. \ref{fig:f4} also shows that the accumulated reward of the proposed method with an SAC is $40 \%$ higher than that with a PPO. This is because the proposed method uses a critic network to generate $V_{\boldsymbol{\psi}}^{\boldsymbol{\pi}_{\boldsymbol{\varphi}}}(\boldsymbol{s}(n))$ and evaluate the value of current state $\boldsymbol{s}(n)$ in order to achieve a stable training process. 

\begin{figure}[tp]
    \centering
\includegraphics[width=.43\textwidth]{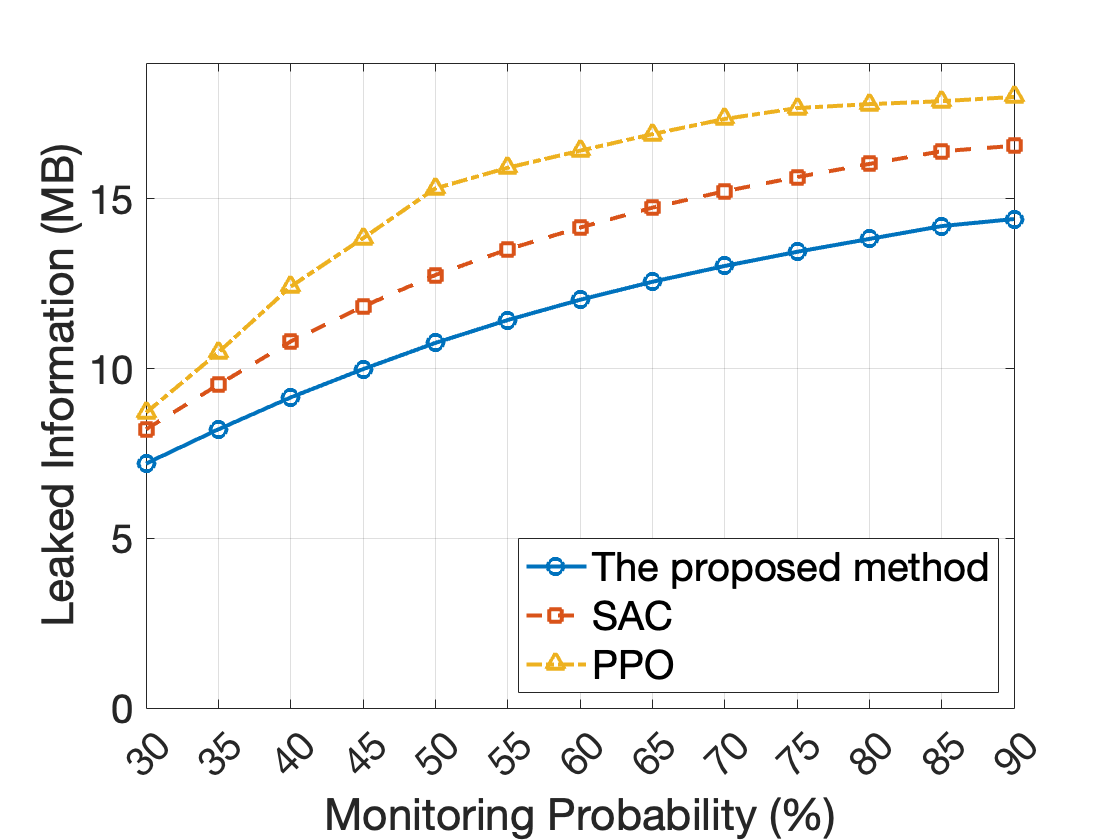}
\vspace{-0.2cm}
    \caption{The amount of information leaked to eavesdroppers as the monitoring probability varies.}
    \label{fig:f5}
      \vspace{-0.5cm}
\end{figure}

Fig. \ref{fig:f5} shows how the monitoring probabilities of each eavesdropper affects the amount of leaked information. From this figure, we see that, with the increase of the monitoring probability, the amount of information leaked from all the considered methods increases. In particular, our proposed method reduces the amount of information leaked to eavesdroppers by $13\%$ and $22 \%$ compared to the SAC and PPO algorithm. This is because the reward $R_{\textrm{C}}\left(\boldsymbol{s}\left(n\right), \boldsymbol{a}\left(n\right)\right)$ from the ICM module helps the agent explore more states during training and find better policies to reduce the amount of leaked information. We also see that, when the monitoring probability increases from $30 \%$ to $90 \%$, the amount of information leaked by the proposed method is less compared to those of the SAC and PPO. This is because the cross-attention module enables the actor network to utilize historical state-action pairs to find important pairs that can make policies to reduce the amount of information leaked to eavesdroppers when the monitoring probability increases.

\begin{figure}[tp]
    \centering
\includegraphics[width=.4\textwidth]{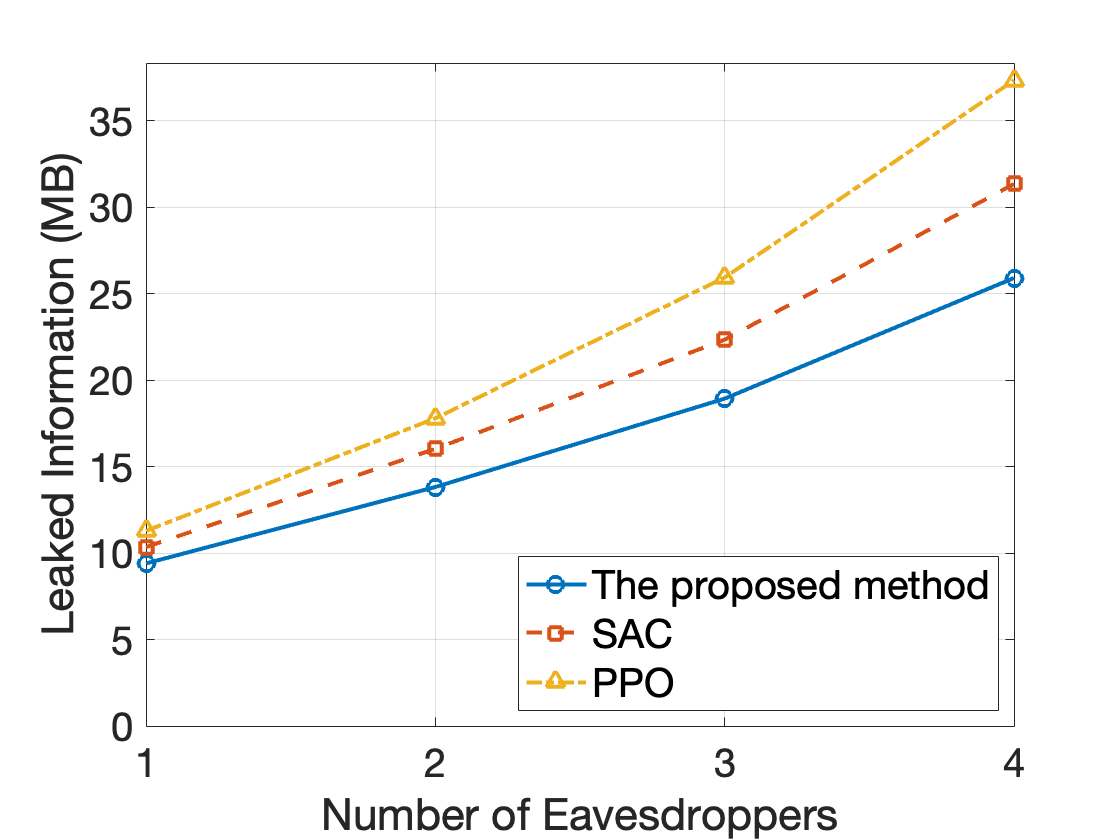}
\vspace{-0.2cm}
    \caption{The amount of information leaked to eavesdroppers as the number of eavesdroppers varies.}
    \label{fig:f6}
      \vspace{-0.5cm}
\end{figure}

In Fig. \ref{fig:f6}, we show the impact of the number of eavesdroppers on the amount of amount of information leaked to eavesdroppers. We see that, when the number of eavesdropper is $1$, all the considered methods have a similar amount of leaked information. This is because all the methods confuse the eavesdropper. However, as the number of eavesdroppers increases, the gap in terms of the amount of information leaked among the proposed method, SAC, and PPO increase. In particular, our method can reduce the amount of information leaked to eavesdroppers by up to $18 \%$ and $30 \%$ compared to SAC and PPO when the network has $4$ eavesdroppers. This is because, with more eavesdroppers, the state space expands rapidly, and it is difficult for traditional RL methods to find optimal policies during training. In contrast, the ICM module in our method can still encourage exploring more states and the CA module can help find better policies that can reduce the amount of information leaked to eavesdroppers.

\begin{figure}[tp]
    \centering
\includegraphics[width=.4\textwidth]{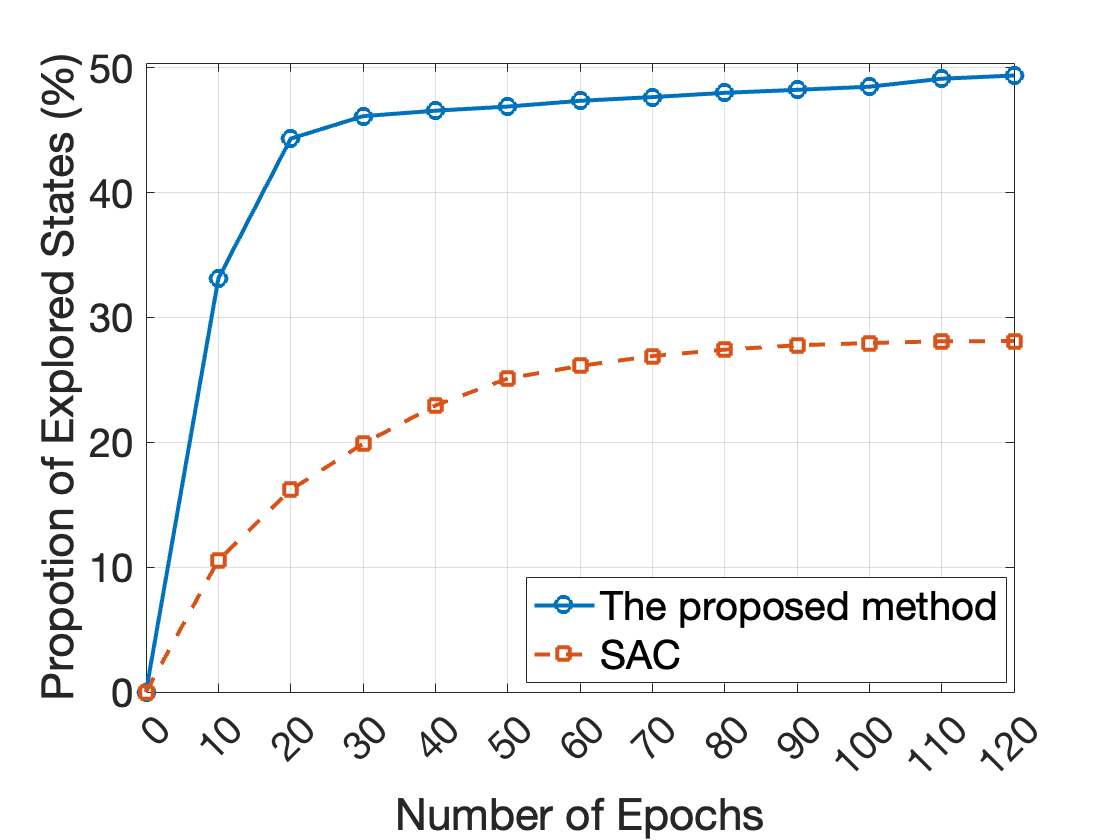}
\vspace{-0.2cm}
    \caption{The state exploration ability of the proposed method and SAC.}
    \label{fig:f7}
      \vspace{-0.5cm}
\end{figure}


Fig.~\ref{fig:f7} shows the state exploration ability of the proposed method and the SAC. From the figure, we see that the proposed method explores $2.5 \times$ states compared to the SAC method within $20$ epochs. This is because the reward $R_{\textrm{C}}(\boldsymbol{s}(n), \boldsymbol{a}(n))$ introduced by the ICM module is large during the first 20 epochs, which motivates the agent to explore unknown states. After the first 20 epochs, the curve flattens because most useful states have already been visited, and the forward and inverse dynamics models are well trained. 

\begin{figure}[tp]
    \centering
\includegraphics[width=.4\textwidth]{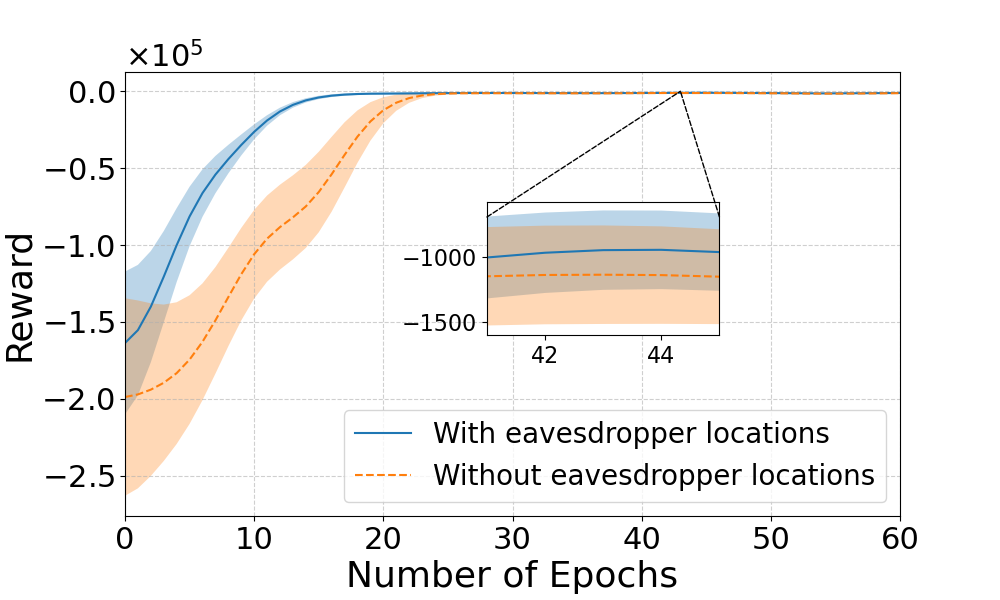}
\vspace{-0.2cm}
    \caption{Convergence of the proposed method without eavesdropper location information.}
    \label{fig:f8}
      \vspace{-0.5cm}
\end{figure}

Fig. \ref{fig:f8} shows the training performance of the proposed method when the server does not know the locations of eavesdroppers. From Fig. \ref{fig:f8}, we see that, without the eavesdropper location information, the proposed method still has similar convergence rate, but the accumulated reward decreases $12 \%$ when the number of epochs is $25$. This is because our method can implicitly estimate the potential locations of eavesdroppers through RL training, thus enabling the agent to assign devices and transmit powers under eavesdropper location uncertainty.

\begin{figure}[t]
    \centering
    \begin{minipage}{0.4\textwidth} 
        \centering
        \includegraphics[width=\linewidth]{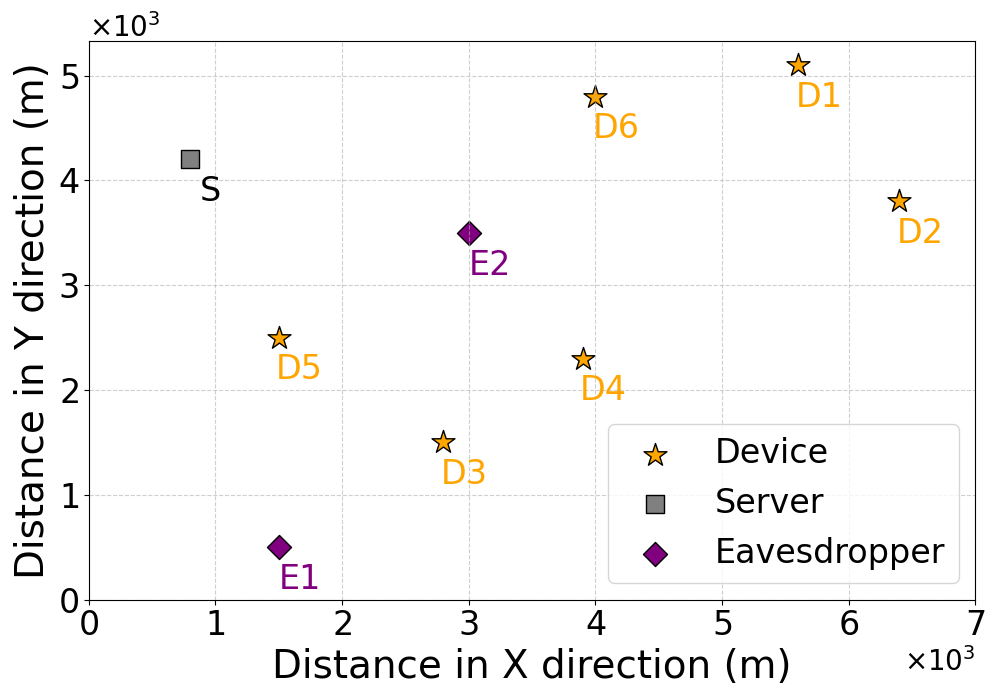}
    \end{minipage}
    \hfill
    \begin{minipage}{0.4\textwidth} 
        \centering
        \includegraphics[width=\linewidth]{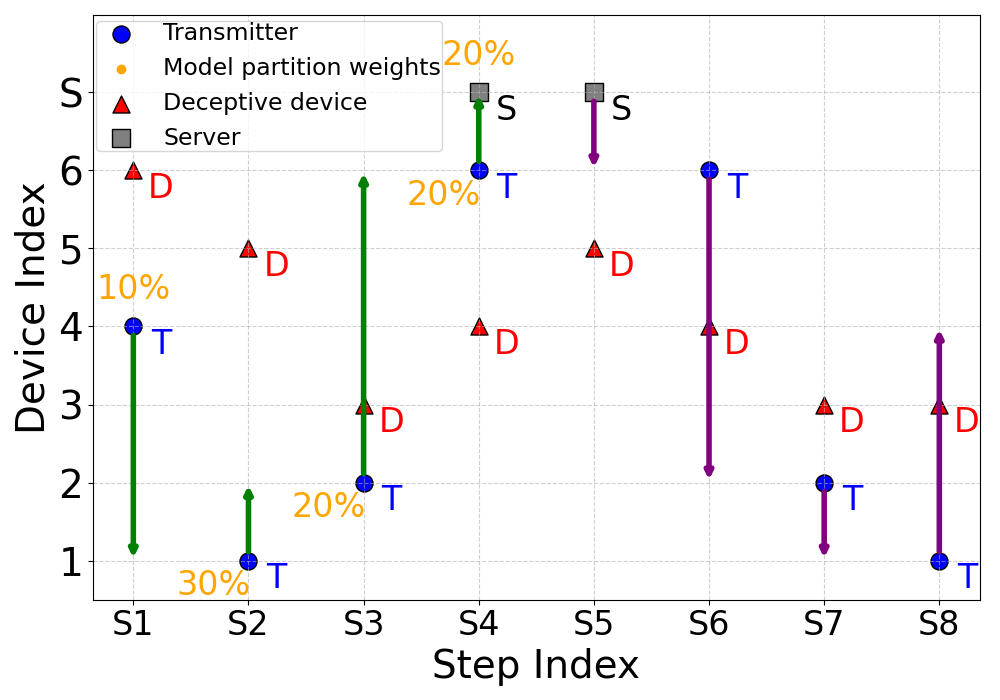}
    \end{minipage}
          \vspace{-0.2cm}
    \caption{Example of the proposed ICM-CA method.}
    \label{fig:f9}
      \vspace{-0.7cm}
\end{figure}

In Fig. \ref{fig:f9}, we show one example of SL training implementation. From this figure, we see that our proposed method selects the devices that are far away from the eavesdroppers as the model training devices, while selects the devices that are close to the eavesdroppers as the deceptive signal transmission devices. 
Fig. \ref{fig:f9} also shows that the algorithm assigns a larger sub-model $\boldsymbol{\theta}_k$ to device $1$ compared to the model assigned to device $4$, since device $1$ is far away from eavesdroppers and has less probability to be monitored. Meanwhile, the algorithm will assign a larger sized sub-model to a model training device (e.g., $D_1$) that is close to other model training devices while maintaining a safe distance from eavesdroppers, which can reduce transmission time while reducing the amount of information leaked to eavesdroppers.

\section{Conclusion}\label{Conclusions}
In this paper, we have proposed an ICM-CA framework to address the challenge of secure split learning in the presence of eavesdroppers. We have formulated an optimization problem that aims to minimize the overall amount of information leaked to eavesdroppers during SL training process, while considering delay and energy consumption constraints. The framework dynamically determines model training device sets, optimizes model splitting, and allocates deceptive signal transmission devices while considering energy, time, and privacy constraints. By leveraging the ICM and CA module, the proposed algorithm encourages the server to explore novel actions and states, and determines the importance of each historical state-action pair, thus improving training efficiency. Simulation results have shown that the ICM-CA framework improves the convergence rate and the accumulated reward compared to the baselines.

\section*{Appendix}
\vspace{-0.5cm}
\subsection{Proof of Theorem \textnormal{\ref{the:1}}}

To prove Theorem \ref{the:1}, we must calculate the probability $\mathbb{P}\left(\epsilon_{s_k,s_k+1}\left(e\right)=s_k\right)$ that eavesdropper $e$ monitors the information from model training device 
$s_k$. Let $S_{d,e} = p_{d} o m_{d,e}^{-2}$ and $S_{s_k,e} = p_{s_k} o m_{s_k,e}^{-2}$. From (\ref{eq:arg_snr}), we see that $\epsilon_{s_k,s_k+1}\left(e\right)=s_k$ only when $\frac{S_{d,e}}{BN_0}<\frac{S_{s_k,e}}{BN_0}, \forall d \in \mathcal{D}_{s_k,s_{k+1}}$. Hence, $\mathbb{P}\left(\epsilon_{s_k,s_k+1}\left(e\right)=s_k\right)$ is
\begin{equation} \label{eq:p1}
    \begin{aligned}
    \mathbb{P}\left(\epsilon_{s_k,s_k+1}\left(e\right)=s_k\right) &= \!\!\!\!\prod_{d \in \mathcal{D}_{s_k,s_{k+1}}}\!\!\!\!\!\!\!\mathbb{P}\left(\frac{S_{d,e}}{BN_0}<\frac{S_{s_k,e}}{BN_0}\right).
\end{aligned}
\end{equation}
Since $\prod_{d \in \mathcal{D}_{s_k,s_{k+1}}}\!\!\!\!\!\!\!\mathbb{P}\left(\frac{S_{d,e}}{BN_0}\!<\!\frac{S_{s_k,e}}{BN_0}\right)\!\!=\!\!\prod_{d \in \mathcal{D}_{s_k,s_{k+1}}}\!\!\!\!\!\!\mathbb{P}\!\left(S_{d,e}\!\!<\!S_{s_k,e}\right)$, we compute $\mathbb{P}\left(S_{d,e}\!\!<\!\!S_{s_k,e}\right)$. Since $S_{d,e}$ and $S_{s_k,e}$ follow exponential distributions with mean $2\!\left(p_d m_{d,e}^{-2} \sigma^2\right)$ and $2\left(p_{s_k} m_{s_k,e}^{-2} \sigma^2\right)$, where $\sigma^2$ is the channel fading variance, the probability density functions (PDF) of $S_{d,e}$ and $S_{s_k,e}$ are:
\begin{equation} \label{eq:pdf}
\begin{aligned}
    f_{S_{d,e}}(s) &= \frac{s}{(C_{d,e} \sigma)^2} e^{-s^2 / (2 (C_d \sigma)^2)}, \quad s \geq 0,\\
    f_{S_{s_k,e}}(s) &= \frac{s}{(C_{s_k,e} \sigma)^2} e^{-s^2 / (2 (C_{s_k} \sigma)^2)}, \quad s \geq 0,
\end{aligned}
\end{equation}
where \( C_{i,e} = p_i m_{i,e}^{-2} \).
Based on $f_{S_{d,e}}(s)$ in (\ref{eq:pdf}), the cumulative distribution functions (CDF) of $S_{d, e}$ is
$F_{S_{d,e}}(s) = 1 - e^{-s^2 / (2 (C_{d,e} \sigma)^2)}$.
Then, the probability $\mathbb{P}\left(S_{d, e} < S_{s_k,e}\right)$ is
\begin{equation} \label{eq:inte_ini}
\begin{aligned}
    &\mathbb{P}(S_{d,e} < S_{s_k,e})\\ &= \int_0^\infty \mathbb{P}(S_{d,e} < s_{s_k,e}) f_{S_{s_k,e}}(s_{s_k,e}) ds_{s_k,e}\\
    &= \int_0^\infty F_{S_{d,e}}(s_{s_k,e})
     \frac{s_{s_k,e}}{(C_{s_k,e} \sigma)^2} e^{-s_{s_k,e}^2 / (2 (C_{s_k,e} \sigma)^2)} ds_{s_k,e} \\
    &=\!\!\! \int_0^\infty \!\!\!\!\left(\!1\!\! -\!\! e\!\!^{-s_{s_k,e}^2 / (2 (C_{d,e} \sigma)^2)}\!\right)\!\!\frac{s_{s_k,e}}{(C_{s_k,e} \sigma)^2}\! e\!^{-s_{s_k,e}^2 / (2 (C_{s_k,e} \sigma)^2)} \!ds_{s_k,e}\\
    &= \int_0^\infty \frac{s_{s_k,e}}{(C_{s_k,e} \sigma)^2} e^{-s_{s_k,e}^2 / (2 (C_{s_k,e} \sigma)^2)} ds_{s_k,e} \\
    &\quad-\!\! \int_0^\infty\!\!\! \frac{s_{s_k,e}}{(C_{s_k,e} \sigma)^2} e^{-s_{s_k,e}^2 \left(\frac{1}{2 (C_{s_k,e} \sigma)^2} + \frac{1}{2 (C_{d,e} \sigma)^2}\right)} ds_{s_k,e}.
\end{aligned}
\end{equation}
Since $\int_0^\infty x e^{-a x^2} dx = \frac{1}{2a}, \quad \forall a > 0$ \cite{zinn2021quantum}, (\ref{eq:inte_ini}) can be rewritten as
\begin{equation} \label{eq:final_P}
\begin{aligned}
    \mathbb{P}(S_{d,e}\! <\! S_{s_k,e}) \!=\!1-\frac{C_{d,e}^2}{C_{d,e}^2 + C_{s_k,e}^2}\!=\! \frac{p_{s_k}m_{s_k,e}^{-2}}{p_{d}m_{t_1,r}^{-2}+p_{s_k}m_{s_k,e}^{-2}}.
\end{aligned}
\end{equation}
Substituting (\ref{eq:final_P}) into (\ref{eq:p1}), we have
\begin{equation}
\begin{aligned}
    \mathbb{P}\left(\epsilon_{s_k,s_k+1}\left(e\right)=s_k\right) 
    &=\!\!\!\!\!\!\! \prod_{d \in \mathcal{D}_{s_k,s_{k+1}}}\!\!\frac{p_{s_k}m_{s_k,e}^{-2}}{p_{d}m_{d,e}^{-2}+p_{s_k}m_{s_k,e}^{-2}}.
\end{aligned}
\end{equation}
Hence, the expectation of the total amount of information leaked to all eavesdroppers is
\begin{equation}
\begin{aligned}
    &\mathbb{E}\left(I_{s_k,s_{k+1}}\right) = \sum_{e \in \mathcal{E}}\mathbb{P}\left(\epsilon_{s_k,s_k+1}\left(e\right)=s_k\right)q_e^{s_k,s_{k+1}} \delta_e^{s_k,s_{k+1}}\left(\boldsymbol{\theta}_k\right) \\
    &=\!\!\sum_{e \in \mathcal{E}}\!\prod_{d \in \mathcal{D}_{s_k,s_{k+1}}}\!\!\!\!\!\frac{p_{s_k}m_{s_k,e}^{-2}}{p_{d}m_{d,e}^{-2}\!+\!p_{s_k}m_{s_k,e}^{-2}}q_e^{s_k,s_{k+1}} \delta_e^{s_k,s_{k+1}}\!\!\left(\boldsymbol{\theta}_k\right).
\end{aligned}
\end{equation}
This completes the proof. \qed

\subsection{Proof of Corollary \ref{col:1}}
To find the optimal $p_{s_k}^*$ and $p_d^*$ to minimize $\mathbb{E}\left(I_{s_k,s_{k+1}}\right)$, we have the following minimization problem
\begin{equation} \label{eq:col_1_mini_ini}
    \min_{\substack{p_{s_k}, p_d}} \quad \sum_{e\in\mathcal{E}} \frac{p_{s_k}\, m_{s_k,e}^{-2}}{p_d\, m_{d,e}^{-2} + p_{s_k}\, m_{s_k,e}^{-2}}q_e^{s_k,s_{k+1}}\delta_e^{s_k,s_{k+1}}\left(\boldsymbol{\theta}_k\right),
\end{equation}
\vspace{-0.5cm}
\begin{align}
    \text{s.t.} \quad & T_{s_k,s_{k+1}}^{\textrm{S}}=\frac{\Gamma\left(\boldsymbol{z}_k\right)}{c_{s_k, s_{k+1}}} \leq B_{\textrm{T}}, \tag{39a} \label{eq:col_1_mini_ini_c1} \\
    & \left(p_{s_k}+p_d\right)T_{s_k,s_{k+1}} \leq B_{\textrm{E}}. \tag{39b} \label{eq:col_1_mini_ini_c2}
\end{align}
Given the data rate $c_{s_k, s_{k+1}}$ in (\ref{eq:data_rate}), (\ref{eq:col_1_mini_ini_c1}) can be written as
\begin{equation} \label{eq:col_1_t2}
    \frac{p_{s_k}m^{-2}_{s_k,s_{k+1}}o}{p_dm_{s_k,d}^{-2}o+BN_0}\geq 2^{\frac{\Gamma\left(\boldsymbol{z}_k\right)}{B_{\textrm{T}}B}}-1.
\end{equation}
Therefore, (\ref{eq:col_1_t2}) can be simplified as
$\xi_0p_{s_k} - \xi_dp_d \geq \chi_1$,
where $\xi_0=m^{-2}_{s_k,s_{k+1}}o$, $\xi_d = {m^{-2}_{s_k,d}o}{\left(2^{\frac{\Gamma\left(\boldsymbol{z}_k\right)}{B_{\textrm{T}}B}}-1\right)}$, and $\chi_1=BN_0\left(2^{\frac{\Gamma\left(\boldsymbol{z}_k\right)}{B_{\textrm{T}}B}}-1\right)$.
From Theorem \ref{the:1}, we see that $\mathbb{E}\left(I_{s_k,s_{k+1}}\right)$ is a monotonically increasing function of $p_{s_k}$ and a decreasing function of $p_d$. Hence, $p_{s_k}^*$ and $p_d^*$ must satisfy the following equations:
\begin{equation} \label{eq:col_1_cst}
    \xi_0p_{s_k}^* - \xi_dp_d^* = \chi_1, \quad p_{s_k}^*+p_d^* = \chi_2,
\end{equation}
where $\chi_2 = \frac{B_{\textrm{E}}}{B_{\textrm{T}}}$.
From (\ref{eq:col_1_cst}), $p_{s_k}^*$ and $p_d^*$ can be derived as
\begin{equation} \label{eq:col_1_sv1}
    p_{s_k}^* = \frac{\chi_1 + \xi_d\chi_2}{\xi_0 + \xi_d}, \quad p_d^* = \frac{\xi_0\chi_2 - \chi_1}{\xi_0 + \xi_d}.
\end{equation}
This completes the proof. \qed

\subsection{Proof of Corollary \ref{col:2}}
To find the optimal $p_{s_k}^*$ and $p_d^*, \forall d \in \mathcal{D}_{s_k,s_{k+1}}$,  we have the following minimization problem
\begin{equation} \label{eq:col_2_mini_ini}
    \min_{\substack{p_{s_k}, p_d}} \quad q_e^{s_k,s_{k+1}}\delta_e^{s_k,s_{k+1}}\left(\boldsymbol{\theta}_k\right) \!\!\!\!\!\!\! \prod_{d \in \mathcal{D}_{s_k, s_{k+1}}}\!\!\!\!\! \left(\frac{p_{s_k}\, m_{s_k,e}^{-2}}{p_d\, m_{d,e}^{-2} + p_{s_k}\, m_{s_k,e}^{-2}}\right),
\end{equation}
\vspace{-0.5cm}
\begin{align}
    \text{s.t.} \quad & T_{s_k,s_{k+1}}^{\textrm{S}}=\frac{\Gamma\left(\boldsymbol{z}_k\right)}{c_{s_k, s_{k+1}}} \leq B_{\textrm{T}}, \tag{43a} \label{eq:col_2_mini_ini_c1} \\
    & \left(p_{s_k}+\!\!\!\!\sum_{d \in \mathcal{D}_{s_k,s_{k+1}}}\!\!\!\!p_d\right)T_{s_k,s_{k+1}} \leq B_{\textrm{E}}, \tag{43b} \label{eq:col_2_mini_ini_c2}
\end{align}
With the assumption in Corollary \ref{col:1}, (\ref{eq:col_1_mini_ini_c1}) can be written as
\begin{equation} \label{eq:col_2_t1}
    \frac{p_{s_k}m^{-2}_{s_k,s_{k+1}}o}{BN_0}\geq 2^{\frac{\Gamma\left(\boldsymbol{z}_k\right)}{B_{\textrm{T}}B}}-1.
\end{equation}
Therefore, (\ref{eq:col_2_t1}) can be simplified as $\xi_0 p_{s_k} \geq \chi_1$.
As $\mathbb{E}\left(I_{s_k,s_{k+1}}\right)$ is a monotonically increasing function of $p_{s_k}$ and a decreasing function of $p_d$, $p_{s_k}^*$ and $p_d^*$ must satisfy
\begin{equation} \label{eq:col_2_cst}
    \xi_0p_{s_k}^* = \chi_1, \quad p_{s_k}^*+\!\!\!\!\sum_{d \in \mathcal{D}_{s_k, s_{k+1}}}\!\!\!\! p_d^* = \chi_2.
\end{equation}
From (\ref{eq:col_2_cst}), we have $p_{s_k}^* = \frac{\chi_1}{\xi_0}$, and
$\sum_{d \in \mathcal{D}_{s_k, s_{k+1}}}\!\!\!\! p_d^* = \chi_2-\frac{\chi_1}{\xi_0}.$
Hence, the minimization problem (\ref{eq:col_2_mini_ini}) can be rewritten as
\begin{equation} \label{eq:col_2_sim}
    \min_{\{p_d\}} \quad q_e^{s_k,s_{k+1}}\delta_e^{s_k,s_{k+1}}\left(\boldsymbol{\theta}_k\right) \!\!\!\!\!\!\!\prod_{d \in \mathcal{D}_{s_k,s_{k+1}}}\!\!\! \frac{\frac{\chi_1}{\xi_0} m_{s_k,e}^{-2}}{p_d m_{d,e}^{-2} + \frac{\chi_1}{\xi_0} m_{s_k,e}^{-2}}.
\end{equation}
\vspace{-0.4cm}
\begin{align}
    \text{s.t.} \quad &
    \sum_{d \in \mathcal{D}_{s_k, s_{k+1}}} p_d = \chi_2 - \frac{\chi_1}{\xi_0}. \label{eq:lcol_2_sim_c1} \tag{46a}
\end{align}
Since $\frac{\frac{\chi_1}{\xi_0} m_{s_k,e}^{-2}}{p_d m_{d,e}^{-2} + \frac{\chi_1}{\xi_0} m_{s_k,e}^{-2}}, \forall d \in \mathcal{D}_{s_k,s_{k+1}}$ is a strictly decreasing function of $p_d$, $\mathbb{E}\left(I_{s_k,s_{k+1}}\right)$ is minimized when all $\frac{\frac{\chi_1}{\xi_0} m_{s_k,e}^{-2}}{p_d m_{d,e}^{-2} + \frac{\chi_1}{\xi_0} m_{s_k,e}^{-2}}, \forall d \in \mathcal{D}_{s_k,s_{k+1}}$ are equal \cite{boyd2004convex}. Let
\begin{equation}
    y_d = \frac{\frac{\chi_1}{\xi_0} m_{s_k,e}^{-2}}{p_d^* m_{d,e}^{-2} + \frac{\chi_1}{\xi_0} m_{s_k,e}^{-2}}= y, \quad \forall d \in \mathcal{D}_{s_k, s_{k+1}} .
\end{equation}
We have
\begin{equation} \label{eq:pd*}
    p_d^* = {\chi_1 m_{s_k,e}^{-2}(1 - y)}/{\xi_0 m_{d,e}^{-2} y}.
\end{equation}
Substituting (\ref{eq:pd*}) into constraint (\ref{eq:lcol_2_sim_c1}), we have
\begin{equation} \label{eq:y}
    y = \frac{\chi_1 m_{s_k,e}^{-2} \sum_{d \in \mathcal{D}_{s_k,s_{k+1}}} m_{d,e}^2}{\chi_1 m_{s_k,e}^{-2} \sum_{d \in \mathcal{D}_{s_k,s_{k+1}}} m_{d,e}^2 + \xi_0 \chi_2 - \chi_1}.
\end{equation}
Substituting (\ref{eq:y}) into (\ref{eq:pd*}), we have
\begin{equation}
    p_d^* = \frac{m^{-2}_{s_k,s_{k+1}}o\frac{B_{\textrm{E}}}{B_{\textrm{T}}}-BN_0\left(2^{\frac{\Gamma\left(\boldsymbol{z}_k\right)}{B_{\textrm{T}}B}}-1\right)}{m^{-2}_{s_k,s_{k+1}}om_{d,e}^{-2}\sum_{d \in \mathcal{D}_{s_k, s_{k+1}}}m_{d,e}^2}.
\end{equation}
This completes the proof. \qed

\bibliography{ref}
\bibliographystyle{IEEEtran}

\end{document}